%% file: arxiv.tex
\documentclass[12pt]{article}
\usepackage{graphicx} % Required for inserting images

\usepackage{etoolbox}
\newtoggle{colt}
\togglefalse{colt}

\include{header}

\include{characters}

\include{macros}

\usepackage[suppress]{color-edits}
\addauthor{ab}{red}
\addauthor{rd}{cyan}
\addauthor{sw}{teal}
\addauthor{mb}{purple}

\usepackage{natbib}

\title{Oracle-Efficient Differentially Private Learning with Public Data}
\usepackage{times}

\usepackage{authblk}
\author[1]{Adam Block}
\author[2]{Mark Bun}
\author[2]{Rathin Desai}
\author[3]{Abhishek Shetty}
\author[4]{Steven Wu}
\affil[1]{MIT}
\affil[2]{Boston University}
\affil[3]{University of California, Berkeley}
\affil[4]{Carnegie Mellon University}

\date{}

\begin{document}

\maketitle

\begin{abstract}
    \input{body/abstract.tex}
\end{abstract}

\tableofcontents

\input{body/intro.tex}
\input{body/prelims.tex}

\input{body/results.tex}

\input{body/analysis.tex}
\input{body/main-classification.tex}

% Acknowledgments---Will not appear in anonymized version
\section*{Acknowledgments}

\input{body/acks.tex}

\bibliographystyle{plainnat}
\bibliography{refs}

\newpage

\appendix

\crefalias{section}{appendix} % uncomment if you are using cleveref

\input{body/app_gaussian_anticoncentration.tex}

\input{body/app_regression_proofs.tex}

\input{body/classification}

\end{document}

%% file: header.tex
\usepackage[margin=0.75in]{geometry}
\usepackage{natbib}
\usepackage{boxedminipage}
\usepackage{multirow,nicefrac}
\usepackage{makecell,upgreek}
\usepackage{footnote}
\usepackage{longtable}
\usepackage{tablefootnote}
\usepackage[T1]{fontenc}
\usepackage{verbatim} 
\usepackage[utf8]{inputenc}

\usepackage{booktabs}   
\usepackage{float}

\iftoggle{colt}{}{
\usepackage[dvipsnames]{xcolor}
}

\iftoggle{colt}{}{
\usepackage{amsthm}

\usepackage[breaklinks=true]{hyperref}
\hypersetup{
	colorlinks=true,
	linkcolor=blue,
	citecolor=blue,
	urlcolor=blue
}
}

\iftoggle{colt}{}{
\usepackage{amsfonts,amssymb,mathrsfs}
\usepackage{mathtools}
\usepackage[capitalise,nameinlink]{cleveref}

\usepackage{bbm}
\usepackage[linesnumbered, ruled,vlined]{algorithm2e}
\usepackage[noend]{algpseudocode}

\Crefname{algocf}{Algorithm}{Algorithms}
\crefname{algocf}{Algorithm}{Algorithms}
\Crefname{algocfline}{Line}{Lines}
\crefname{algocfline}{Line}{Lines}

\usepackage{appendix}

\theoremstyle{plain}
	\newtheorem{theorem}{Theorem}

	\newtheorem{lemma}{Lemma}

	\newtheorem{corollary}{Corollary}
	\newtheorem{proposition}{Proposition}

	\theoremstyle{definition}

	\newtheorem{definition}{Definition}

	\newtheorem{remark}{Remark}

}

\usepackage{autonum}

%% file: characters.tex
\def\ddefloop#1{\ifx\ddefloop#1\else\ddef{#1}\expandafter\ddefloop\fi}
\def\ddef#1{\expandafter\def\csname bb#1\endcsname{\ensuremath{\mathbb{#1}}}}
\ddefloop ABCDEFGHIJKLMNOPQRSTUVWXYZ\ddefloop

\def\ddefloop#1{\ifx\ddefloop#1\else\ddef{#1}\expandafter\ddefloop\fi}
\def\ddef#1{\expandafter\def\csname frak#1\endcsname{\ensuremath{\mathfrak{#1}}}}
\ddefloop ABCDEFGHIJKLMNOPQRSTUVWXYZ\ddefloop

\def\ddefloop#1{\ifx\ddefloop#1\else\ddef{#1}\expandafter\ddefloop\fi}
\def\ddef#1{\expandafter\def\csname fr#1\endcsname{\ensuremath{\mathfrak{#1}}}}
\ddefloop ABCDEFGHIJKLMNOPQRSTUVWXYZ\ddefloop

\def\ddefloop#1{\ifx\ddefloop#1\else\ddef{#1}\expandafter\ddefloop\fi}
\def\ddef#1{\expandafter\def\csname eul#1\endcsname{\ensuremath{\EuScript{#1}}}}
\ddefloop ABCDEFGHIJKLMNOPQRSTUVWXYZ\ddefloop

\def\ddefloop#1{\ifx\ddefloop#1\else\ddef{#1}\expandafter\ddefloop\fi}
\def\ddef#1{\expandafter\def\csname scr#1\endcsname{\ensuremath{\mathscr{#1}}}}
\ddefloop ABCDEFGHIJKLMNOPQRSTUVWXYZ\ddefloop

\def\ddefloop#1{\ifx\ddefloop#1\else\ddef{#1}\expandafter\ddefloop\fi}
\def\ddef#1{\expandafter\def\csname b#1\endcsname{\ensuremath{\mathbf{#1}}}}
\ddefloop ABCDEGHIJKLMNOPQRSTUVWXYZ\ddefloop

\def\ddefloop#1{\ifx\ddefloop#1\else\ddef{#1}\expandafter\ddefloop\fi}
\def\ddef#1{\expandafter\def\csname bhat#1\endcsname{\ensuremath{\hat{\mathbf{#1}}}}}
\ddefloop ABCDEFGHIJKLMNOPQRSTUVWXYZ\ddefloop

\def\ddefloop#1{\ifx\ddefloop#1\else\ddef{#1}\expandafter\ddefloop\fi}
\def\ddef#1{\expandafter\def\csname btil#1\endcsname{\ensuremath{\tilde{\mathbf{#1}}}}}
\ddefloop ABCDEFGHIJKLMNOPQRSTUVWXYZ\ddefloop

\def\ddefloop#1{\ifx\ddefloop#1\else\ddef{#1}\expandafter\ddefloop\fi}
\def\ddef#1{\expandafter\def\csname bst#1\endcsname{\ensuremath{\mathbf{#1}^\star}}}
\ddefloop ABCDEFGHIJKLMNOPQRSTUVWXYZ\ddefloop

\def\ddefloop#1{\ifx\ddefloop#1\else\ddef{#1}\expandafter\ddefloop\fi}
\def\ddef#1{\expandafter\def\csname bst#1\endcsname{\ensuremath{\mathbf{#1}^\star}}}
\ddefloop abcdeghijklmnopqrstuvwxyz\ddefloop

\def\ddefloop#1{\ifx\ddefloop#1\else\ddef{#1}\expandafter\ddefloop\fi}
\def\ddef#1{\expandafter\def\csname bhat#1\endcsname{\ensuremath{\hat{\mathbf{#1}}}}}
\ddefloop abcdefghijklmnopqrstuvwxyz\ddefloop

\def\ddefloop#1{\ifx\ddefloop#1\else\ddef{#1}\expandafter\ddefloop\fi}
\def\ddef#1{\expandafter\def\csname b#1\endcsname{\ensuremath{\mathbf{#1}}}}
\ddefloop abcdeghijklnopqrstuvwxyz\ddefloop

\def\ddefloop#1{\ifx\ddefloop#1\else\ddef{#1}\expandafter\ddefloop\fi}
\def\ddef#1{\expandafter\def\csname barb#1\endcsname{\ensuremath{\bar{\mathbf{#1}}}}}
\ddefloop abcdefghijklmnopqrstuvwxyz\ddefloop

\def\ddef#1{\expandafter\def\csname c#1\endcsname{\ensuremath{\mathcal{#1}}}}
\ddefloop ABCDEFGHIJKLMNOPQRSTUVWXYZ\ddefloop
\def\ddef#1{\expandafter\def\csname h#1\endcsname{\ensuremath{\widehat{#1}}}}
\ddefloop ABCDEFGHIJKLMNOPQRSTUVWXYZ\ddefloop
\def\ddef#1{\expandafter\def\csname hc#1\endcsname{\ensuremath{\widehat{\mathcal{#1}}}}}
\ddefloop ABCDEFGHIJKLMNOPQRSTUVWXYZ\ddefloop
\def\ddef#1{\expandafter\def\csname t#1\endcsname{\ensuremath{\widetilde{#1}}}}
\ddefloop ABCDEFGHIJKLMNOPQRSTUVWXYZ\ddefloop
\def\ddef#1{\expandafter\def\csname tc#1\endcsname{\ensuremath{\widetilde{\mathcal{#1}}}}}
\ddefloop ABCDEFGHIJKLMNOPQRSTUVWXYZ\ddefloop

%% file: macros.tex
\newcommand{\norm}[1]{\left\|#1\right\|}
\newcommand{\abs}[1]{\left|#1\right|}
\newcommand{\inprod}[2]{\left\langle #1, #2 \right\rangle}

\newcommand{\rr}{\mathbb{R}}
\newcommand{\ee}{\mathbb{E}}
\newcommand{\pp}{\mathbb{P}}

\newcommand{\algfont}[1]{\mathsf{#1}}
\newcommand{\erm}{\algfont{ERM}}

\newcommand{\fbar}{\bar{f}}
\newcommand{\fhat}{\widehat{f}}
\newcommand{\fstar}{f^\star}
\newcommand{\ftil}{\widetilde{f}}

\newcommand{\cDtil}{\widetilde{\cD}_{x}}

\newcommand{\Lap}{Lap}

\newcommand{\tstar}{t^\star}

\newcommand{\mutil}{\widetilde{\mu}}
\DeclareMathOperator*{\argmin}{argmin}

\renewcommand{\epsilon}{\varepsilon}
\newcommand{\ind}[1]{^{(#1)}}
\newcommand{\poly}{\textsf{poly}}
\newcommand{\perturb}{\algfont{Perturb}}
\newcommand{\zstar}{z^\star}

\newcommand{\ferm}{f_{\mathsf{ERM}}}
\newcommand{\rbar}{\overline{r}}
\newcommand{\cGbar}{\overline{\cG}}
\newcommand{\fat}{\mathsf{fat}}
\newcommand{\Thetatil}{\widetilde{\Theta}}
\newcommand{\Omegatil}{\widetilde{\Omega}}
\newcommand{\cLtil}{\widetilde{\cL}}
\newcommand{\vc}{\mathsf{vc}}
\DeclareMathOperator{\conv}{conv}

%% file: body/abstract.tex
Due to statistical lower bounds on the learnability of many function classes under privacy constraints, there has been recent interest in leveraging public data to improve the performance of private learning algorithms. In this model, algorithms must always guarantee differential privacy with respect to the private samples while also ensuring learning guarantees when the private data distribution is sufficiently close to that of the public data.  Previous work has demonstrated that when sufficient public, unlabelled data is available, private learning can be made statistically tractable, but the resulting algorithms have all been computationally inefficient.  In this work, we present the first computationally efficient, algorithms to provably leverage public data to learn privately whenever a function class is learnable non-privately, where our notion of computational efficiency is with respect to the number of calls to an optimization oracle for the function class.  In addition to this general result, we provide specialized algorithms with improved sample complexities in the special cases when the function class is convex or when the task is binary classification.

%% file: body/intro.tex
\section{Introduction}

Differential privacy (DP)~\citep{DworkMNS06} is a standard guarantee of individual-level privacy for statistical data analysis. Algorithmic research on differential privacy aims to understand what statistical tasks are compatible with the definition, and at what cost, e.g., in terms of sample complexity or computational efficiency. Unfortunately, it is known that some tasks may become more expensive or outright impossible to conduct with differential privacy. For example, in the setting of binary classification, there is no differentially private algorithm for solving the simple problem of learning a one-dimensional classifier over the real numbers~\citep{bun2015differentially, AlonLMM19}.

Motivated in part by such barriers to full-fledged private learning, many papers have considered relaxing the private learning model to allow the use of auxiliary ``public'' data~\cite{beimel2014learning, BassilyMA19, bassily2020learning, bassily2022private, bassily2023principled, kairouz2021nearly, amid2022public,  lowy2023optimal}. Such data may be available if individuals can voluntarily opt-in to share or sell their information to enable a particular task. Alternatively, a data analyst might have background knowledge about the underlying data distribution from the results of previous analyses, or hold a plausible generative model for it. Situations like this are captured by the semi-private learning model, first introduced by \citet{beimel2014learning} and subsequently studied by \citet{BassilyMA19}. In this model, a learning algorithm is given $n$ ``private'' samples from a joint distribution $\mathcal{D}$ over example-label pairs, as well as $m$ \emph{unlabeled} ``public'' samples from the same marginal distribution over examples. The algorithm must be differentially private with respect to its private dataset, but can depend arbitrarily on its public samples. For learning a binary classifier over a class $\cF$ with a VC-dimension $\vc(\cF)$, these papers showed that in the presence of $O(\vc(\cF))$ public unlabeled samples, every concept class $\cF$ is agnostically learnable with $O(\vc(\cF))$ private labeled samples, matching what is achievable without privacy guarantees at all.

While these results essentially resolve the statistical complexity of semi-private learning, they do not address the question of computational efficiency. These algorithms proceed by drawing enough public samples to construct a cover for the class $\cF$ with respect to the target marginal distribution on examples, and then using the exponential mechanism \citep{McSherryT07} to select a hypothesis from this cover that fits the private dataset. As the size of this cover is exponential in $\vc(\cF)$, constructing it explicitly is computationally expensive. This paper aims to address the following question: \emph{Is such computational overhead really necessary if $\cF$ exhibits additional structure that make non-private learning tractable?}

In this work, we give new semi-private learners for regression and classification classes $\cF$ that are efficient whenever fast non-private algorithms are available. More specifically, our main result is generic semi-private algorithms for regression and classification that are \emph{oracle-efficient} in that they run in polynomial time given an oracle solving the \emph{non-private} empirical risk minimization problem for $\cF$, and have sample complexity polynomial in the usual parameters such as Gaussian complexity and VC dimension.

\begin{theorem}[Informal version of~\Cref{thm:general_ftpl_ub}]\label{informalthm}
	Fix a function class $\cF: \cX \to [-1,1]$ with Gausssian complexity $\sqrt{d}$. Then there is an oracle-efficient, $(\varepsilon, \delta)$-differentially private algorithm (\Cref{alg:general_alg}) using $\poly(d)$ labeled private samples, unlabeled public samples, and calls to an empirical risk minimization oracle for $\cF$ that learns an approximately optimal predictor $\hat f \in \cF$.
\end{theorem}

\swedit{While \Cref{informalthm} captures extremely broad learning settings}, the polynomials governing its sample complexity are rather large. We identify several important cases in which the sample complexity can be improved and the number of oracle calls is only 2. In the case where the function class $\cF$ is convex, we give a variant of \Cref{alg:general_alg}, inspired by follow-the-regularized-leader, with significantly improved sample complexity as a function of the desired error (\Cref{thm:cvx_ftrl}, \Cref{alg:ftrl_alg}). Finally, in the special case of binary classification (i.e., Boolean $\cF$ under the 0-1 loss), we give a completely different oracle-efficient algorithm with improved sample complexity (\Cref{thm:classification-upper-bound}, \Cref{alg:RRSPM}), \swedit{which require the private sample size to grow at the rate of $O((\vc(\cF))^2)$.} Our results in the binary classification setting can also be viewed as an {extension} of \cite{neel2019use}, which gives oracle-efficient private learners for structured function classes $\cF$ that have a small \emph{universal identification set} \citep{goldman1993exact}. Our results relax this stringent combinatorial condition by leveraging a small public unlabelled dataset, which allows us to design an oracle-efficient private learner for any function class $\cF$ with bounded VC-dimension.

In fact, our results also address a somewhat more general setting than the semi-private model described so far. Specifically, our results automatically handle the setting where there may be a bounded distribution shift between the public and private data. In particular, all of our results hold as long as the public unlabelled data distribution and the private marginal distribution over the feature space have a density ratio bounded by $\sigma$.\footnote{In fact, this condition can be further relaxed to only assuming the two distributions have a bounded $f$-divergence.} The standard semi-private setting corresponds to the special case where $\sigma = 1$. Taking this view, we can interpret our results as oracle-efficient private learning in the \emph{smoothed learning} setting. Our algorithms achieve accurate learning provided that the private marginal distribution does not deviate too much from a public reference distribution. However, our privacy guarantees hold even if the private data distribution has unbounded distribution shift from the public data.

\subsection{Related Work}

Our work brings together ideas and techniques from multiple literatures.

\paragraph{Oracle efficiency in private and online learning.} Our notion of oracle-efficiency is standard in (theoretical) machine learning to model reductions in a world where worst-case hardness abounds, but optimization heuristics (e.g., integer programming solvers, non-convex optimization) often enjoy success. Within the differential privacy literature, oracle-efficient algorithms are known for binary classification with classes $\cF$ that admits small universal identification sets~\citep{neel2019use}, synthetic data generation~\citep{gaboardi2014dual, nikolov2013geometry, neel2019use, vietri2020new}, and certain types of non-convex optimization problems~\citep{Neel0VW20}. Oracle-efficiency is also \swedit{well-established approach} in \emph{online learning}~\citep{kalai2005efficient,hazan2016computational,kozachinskiy2023simple,haghtalab2022oracle,block2022efficient,block2023sample,block2023oracle,block2023smoothed}, an area with deep connections to differential privacy~\citep{AlonLMM19,abernethy2019online,bun2020equivalence,ghazi2021sample}. 
Indeed, our new semi-private learning algorithm \Cref{alg:general_alg} adapts a follow-the-perturbed-leader inspired algorithm~\citep{block2022smoothed} from the setting of \emph{smoothed} online learning.

\paragraph{DP learning and release with public (unlabelled) data} Our results contribute to a long line of theoretical work that leverages public data for private data analysis. In particular, our work provides general computationally efficient algorithms (in the oracle efficiency sense) for semi-private learning \citep{beimel2014learning, BassilyMA19}, \mbedit{whereas prior work focused only on sample complexity. Several recent papers~\cite{bassily2022private, bassily2023principled} developed efficient algorithms for private learning with domain adaptation from a public source. That work accommodates a more general notion of distribution shift than ours, but makes essential use of \emph{labeled} public data, as well as handling only restricted concept classes or loss functions. There} has also been work that leverages public data to remove statistical barriers in private query release \citep{BassilyCMNUW20} and density estimation~\cite{bie2022private, ben2023private}. In parallel, there has also been a large body of empirical work on using public data for private learning, including \cite{PapernotSMRTE18, yu2022differentially, Golatkar2022, ZhouW021}, and private synthetic data, including \cite{liu2021leveraging, LiuVW21}. \swcomment{probably need to move the cites on the domain adaptation her}

\paragraph{Smoothed Analysis in Online Learning.}  Smoothed analysis was pioneered in \citet{spielman2004smoothed} for the purpose of explaining the empirical success of algorithms whose worst-case behavior is provably intractable.  More recently, the framework has come to online learning \citep{rakhlin2011online,haghtalab2020smoothed,haghtalab2022smoothed,block2022smoothed,haghtalab2022oracle,block2022efficient,block2023oracle,block2023smoothed} in order to circumvent the strong statistical \citep{rakhlin2015online} and computational \citep{hazan2016computational} lower bounds that worst-case data can induce.  
The assumption of smoothness has also been used in learning problems more broadly \citep{durvasula2023smoothed,cesa2023repeated} and its assumptions have been relaxed \citep{block2023sample}.

%% file: body/prelims.tex
\section{Preliminaries}
In this section, we formally introduce our setting. Let $\cX$ denote the feature space and $\cY$ be the label space. In general, we consider $\cY= [-1, 1]$, but in the special case of binary classification setting, we have $\cY=\{0, 1\}$. In general, we study learning algorithms $\cA$ that map a dataset $\cD$ with $n$ examples from $\cX\times \cY$ to a predictor in a function class $\cF$. We require $\cA$ to satisfy \emph{differential privacy}, defined below.
\begin{definition}[Differential Privacy \cite{DworkMNS06}]\label{def:dp}
    Let $\cA: (\cX \times \cY)^n \to \cF$ be a randomized algorithm and $\cD, \cD' \in (\cX \times \cY)^n$ be data sets.  We say that $\cD$ and $\cD'$ are \emph{neighboring} if $\abs{\cD \setminus \cD'} = \abs{\cD' \setminus \cD} \leq 1$, i.e. they differ in at most one datum.  We say that $\cA$ is $(\epsilon, \delta)$-\emph{differentially private} if for all neighboring datasets $\cD, \cD'$, and for all measurable $\cG \subset \cF$, it holds that
    \begin{align}
        \pp\left( \cA(\cD) \in \cG \right) \leq e^\epsilon \cdot \pp\left( \cA(\cD') \in \cG \right) + \delta.
    \end{align}
    If $\delta = 0$, we say that $\cA$ is $\epsilon$-(purely) differentially private.
\end{definition}

As defined, it is trivial to construct algorithms that are differentially private by outputting functions independent of the data set; for an algorithm to be useful, however, we also require that it learns in a meaningful sense.  Thus, in the context of learning, we consider the following accuracy desideratum.

\begin{definition}\label{def:pac}
    Let $\cA: (\cX \times \cY)^n \to \cF$ be a randomized algorithm.  We say that $\cA$ is an $(\alpha, \beta)$-learner with respect to a measure $\nu$ on $\cX \times \cY$ and loss function $L: \cF \to [-1,1]$ if, for $\cD$ sampled independently from $\nu$, it holds that 
    \begin{align}
        \pp\left(L(\cA(\cD)) \leq \inf_{f \in \cF} L(f) + \alpha  \right) \geq 1 - \beta.
    \end{align}
   For regression problems, we consider the loss function $L$ to be induced by a function $\ell: \cY \times \cY \to [0,1]$, convex and $\lambda$-Lipschitz in the first argument, such that $L(f) = \ee_{(X,Y) \sim \nu} \left[ \ell(f(X), Y) \right]$.
\end{definition}
For simplicity, we will denote the empirical loss on a data set $\cD$ as $L_\cD(f) = \frac 1{n} \cdot \sum_{(X_i, Y_i) \in \cD} \ell(f(X_i), Y_i)$.  
We emphasize that in contradistinction to the standard notion of PAC-learnability \citep{valiant1984theory}, our requirement is weaker in that we only require \emph{distribution-dependent} learning, i.e., the algorithm $\cA$ is allowed to depend on $\nu$ in some to-be-specified way.  This is necessary in our setting as it is well known that distribution-independent differentially private PAC learning is \swedit{possible only for very restricted classes of functions $\cF$ with bounded Littlestone dimension \citep{AlonLMM19,bun2020equivalence}.} 
To make private learning statistically tractable for broader classes of functions, we consider the following restriction on $\nu$:
\begin{definition}\label{def:smooth}
    Given a measure $\mu \in \Delta(\cX)$ and a parameter $\sigma \in (0,1]$, we say that $\nu_x$ is $\sigma$-smooth with respect to $\mu$ if $\norm{\frac{d\nu_x}{d \mu}}_\infty \leq \frac 1\sigma$.  We suppose that the learner has access to $m$ samples $Z_1, \dots, Z_m \sim \mu$ that are independent of each other and the training data $\cD$ and thus $\cA$ may depend on these samples.
\end{definition}
We remark that \Cref{def:smooth} can be significantly relaxed by assuming only that $D_f(\nu_x || \mu) \leq \frac 1\sigma$ as in \citet{block2023sample}, where $D_f(\cdot || \cdot)$ is a sufficiently strong $f$-divergence\footnote{These divergences include the well-known KL divergence and Renyi divergence.  For a comprehensive introduction to $f$-divergences, see \citet{polyanskiy2022}.}.  In this case, the statistical rates presented below will be worse and depend on $f$, but the algorithms and privacy guarantees will remain unchanged.  Critically, we do not require that our algorithms are private with respect to $Z_1, \dots, Z_m$, which we treat as public, unlabelled data.   

Our primary focus is to design computationally efficient private learners. To this end, we suppose access to the following ERM oracle:
\begin{definition}\label{def:erm_oracle}
    Given a function class $\cF: \cX \to \rr$, a data set $\cD = \{x_1, \dots, x_m \} \subset \cX$ and loss functions $\ell_1, \dots, \ell_m: \rr \to \rr$, we define the empirical risk minimization oracle $\erm: \cF \to \rr$ such that
    \begin{align}
        \erm(\cF, \cL_{\cD}) \in \argmin_{f \in \cF} \cL_{\cD}(f),
    \end{align}
    where $\cL_{\cD}(f) = \sum_{x_i \in \cD} \ell_i(f(x_i))$.
\end{definition}
ERM oracles are standard computational models in many learning domains such as online learning \citep{kalai2005efficient,hazan2016computational,block2022smoothed,haghtalab2022oracle} and Reinforcement Learning  \citep{foster2020beyond,foster2021statistical,mhammedi2023representation,mhammedi2023efficient}.  Assuming access to $\erm$ allows us to disentangle the computational challenges of optimizing over specific function classes from the specific challenge of differentially private learning as well as to avoid the well-known intractability results for nonconvex optimization \citep{blum1988training} that do not accurately reflect the realities of modern optimization techniques (e.g., integer program solvers, SGD).  \abedit{We remark that applying \citet[Theorem 8]{neel2019use} gives a black-box robustification procedure for purely private, oracle-efficient algorithms, which ensures that the privacy guarantees continue to hold even when the oracle may fail to optimize the objective. In particular, \Cref{alg:ftrl_alg,alg:RRSPM} below, when run in their pure DP forms can be made robust at a minimal cost on accuracy. We defer to \citet{neel2019use} for further discussion on this topic.}
 
It is well-known that even absent differential privacy guarantees, learning arbitrary function classes is impossible; we now introduce the notions of complexity that are relevant to our results.  We begin with the standard notion of VC dimension:
\begin{definition}\label{def:vc}
    Let $\cF: \cX \to \left\{ 0, 1 \right\}$ be a function class.  We say that a set of points $x_1, \dots, x_d \in \cX$ shatters $\cF$ if for all $\epsilon_{1:d} \in \left\{ 0, 1 \right\}^d$, there is some $f_\epsilon$ such that $f_\epsilon(x_i) = \epsilon_i$ for all $i$.  The VC dimension of $\cF$, denoted $\vc(\cF)$, is the largest $d$ such that there exists a set of $d$ points that shatters $\cF$.
\end{definition}
In addition to VC dimension, we also use the Gaussian complexity of a function class:
\begin{definition}\label{def:gaussianccomplexity}
    Let $\cF: \cX \to [-1,1]$ be a function class and $x_1, \dots, x_m \in \cX$ be arbitrary points.  We let $\omega_m: \cF \to \rr$ be the canonical Gaussian process on $\cF$, i.e.,
    \begin{align}\label{eq:gp_def}
        \omega_m(f) = \frac 1{\sqrt m} \cdot \sum_{ i =1}^m \xi_i \cdot f(x_i),
    \end{align}
    where $\xi_i$ are independent standard Gaussians.  We define the (data-dependent) Gaussian complexity of $\cF$ to be $\ee\left[ \sup_{f \in \cF} \omega_m(f) \right]$, the average Gaussian complexity as $\cG_m(\cF) = \ee_Z \ee[\sup_{f \in \cF} \omega_m(f)]$, and the worst-case Gaussian complexity of $\cF$ to be $\cGbar_m(\cF) = \sup_{x_1, \dots, x_m \in \cX} \ee\left[ \sup_{f \in \cF} \omega_m(f) \right]$.
\end{definition}
Both $\vc(\cF)$ and $\cG_m(\cF)$ are well known measures of complexity from learning theory and their relationships to other notions of complexity like covering number are well-understood \citep{mendelson2003entropy,wainwright2019high,van2014probability}.  In particular, it is well-known that $\cG_m(\cF) = O(\sqrt{\vc(\cF)})$ \citep{dudley1969speed,mendelson2003entropy} and that standard PAC-learning is possible if and only if $\cGbar_m(\cF) = o(\sqrt m)$ \citep{wainwright2019high,van2014probability}.  We remark that different texts use different scalings for $\cG_m(\cF)$, with some replacing the $m^{-1/2}$ factor in \eqref{eq:gp_def} with $m^{-1}$ and others omitting it entirely; our choice of scaling is motivated by the fact that a natural complexity measure for many (Donsker \citep{wainwright2019high}) function classes that our algorithms depend on is $\sup_m \cGbar_m(\cF)$, which is most compactly represented with the present scaling.

\paragraph{Notation.}  We always reserve $\pp$ and $\ee$ for probability and expectation with respect to measures that are clear from the context. We denote by $\Delta(\cX)$ the space of measures on some $\cX$ and for any $\mu \in \Delta(\cX)$ we let $\norm{\cdot}_\mu$ denote the $L^2(\mu)$ norm, i.e., $\norm{f}_\mu^2 = \ee_{Z \sim \mu}[f(Z)^2]$.  Similarly, for $m$ points $Z_1, \dots, Z_m \in \cX$, we let $\norm{\cdot}_m$ denote the empirical $L^2$ norm on these points so that $\norm{f}_m^2 = m^{-1} \cdot \sum_{i = 1}^m f(Z_i)^2$.  We reserve $\omega_m$ for the canonical empirical Gaussian process on $\cF$ as in \eqref{eq:gp_def} and $\cL$ for a functional on $\cF$.

%% file: body/results.tex
\section{Algorithms for Differentially Private Learning}\label{sec:results}

In this section, we provide a general template for constructing differentially private learning algorithms with public data and instantiate this template with two oracle-efficient algorithms.  Our first algorithm applies to arbitrary bounded function classes, whereas the second algorithm only applies to convex classes but has an improved sample complexity.  Our general template is broken into the following two steps:
\begin{enumerate}
    \item Use $\erm$ (cf. \Cref{def:erm_oracle}) and the public data to construct an initial estimate $\fbar$ that is a good learner and satisfies stability with respect to $\norm{\cdot}_m$.
    \item Output $\fhat$ as the function that minimizes $\norm{f - \fbar - \gamma \cdot \zeta}_m$, where $\gamma \geq 0$ is a scale and $\zeta = (\zeta_1, \dots, \zeta_m)$ is a vector of independent random variables sampled according to some distribution $\cQ$.
\end{enumerate}
\begin{algorithm}
    \begin{algorithmic}[1]
        \State \textbf{Input } Function $\fbar \in \cF$, distribution $\cQ \in \Delta(\rr)$, scale $\gamma \geq 0$, public data $\cDtil = \left\{ Z_1, \dots, Z_m \right\}$.
        \State{} \textbf{Sample} $\zeta_1, \dots, \zeta_m \sim \cQ$.
        \State{} \textbf{Define } $R(f) = \norm{f - \fbar - \gamma \cdot \zeta}_m^2$.
        \State{} \textbf{Output } $\fhat = \erm(R, \cF)$.
    \end{algorithmic}
    \caption{$\perturb$: An algorithm for perturbing a function with noise on public data.}
    \label{alg:output_perturb}
\end{algorithm}
The second step is accomplished through \Cref{alg:output_perturb} and is the same across our algorithms.  The first step, however, is algorithm-specific and is the primary factor affecting the sample complexity.   The intuition for our template is as follows.  We need to show that $\fhat$ is both a learner and is differentially private.  To see why the template produces a good learner, note that if $Z_i \sim \mu$ are independent and $m$ is sufficiently large, then $\norm{\cdot}_m \approx \norm{\cdot}_{\mu}$.  Thus if $\gamma$ is small, then $\norm{\fbar - \fhat}_{\mu} \ll 1$ and $\abs{\ee_\mu\left[ L(\fhat) \right] - \ee_\mu[L(\fbar)]} \ll 1$ whenever $\ell$ is Lipschitz.  By smoothness, a similar guarantee holds for expectations with respect to $\nu$ and thus $\fhat$ is a good learner.  To see why $\fhat$ is differentially private, note that by choosing $\cQ$ to be a standard Gaussian, we can ensure that the likelihood ratios of choosing $\fhat$ given $\fbar$ verse $\fbar'$ are controlled by $\norm{\fbar - \fbar'}_m$.  Thus, if $\fbar$ is stable with respect to $\norm{\cdot}_m$, then $\fhat$ will be differentially private.  

We now make the above intuition precise by instantiating this template in our most general setting in \Cref{alg:general_alg}.  We construct $\fbar$ by running $\erm$ on a perturbed version of the empirical risk minimization problem and then averaging.  Specifically, for $j \in [J]$, we define $\cL\ind{j}: \cF \to \rr$ as a sample path of a noncentred Gaussian process in \eqref{eq:omega_def} and let $\fbar_j$ denote the minimizer of $\cL\ind{j}$ over $\cF$.  We then output $\fbar$ as the average of $\fbar_1, \dots, \fbar_J$.  We present motivation for the particular choice of $\fbar$, as well as the analogue in \Cref{alg:ftrl_alg}, in the subsequent section.  The following theorem shows that if $\cQ$ is chosen correctly, this algorithm is an oracle-efficient differentially private learner whenever $\nu_x$ is $\sigma$-smooth with respect to $\mu$.
\begin{algorithm} \SetAlgoNoLine
    \begin{algorithmic}[1]
    \State \textbf{Input } Oracle $\erm$, perturbation parameter $\eta > 0$, public data set $\cDtil = \left\{ Z_1, \dots, Z_m \right\}$, private data set $\cD = \left\{ (X_i, Y_i) | 1 \leq i \leq n \right\}$, function class $\cF$, loss function $\ell$, noise level $\gamma > 0$, number of iterations $J \in \bbN$, noise distribution $\cQ \in \Delta(\rr)$. 
    \iftoggle{colt}{\\}{}
    \For{$j = 1, 2, \dots, J$}{
        \State \indent \textbf{Sample} $\xi_1^{(j)}, \dots, \xi_m^{(j)} \sim \cN(0, 1)$.
        \State \indent \textbf{Define } $\omega_m\ind{j}: \cF \to \rr$ such that
        \begin{equation}
            \omega_m^{(j)}(f) = \frac 1{\sqrt m} \cdot \sum_{i = 1}^m \xi_i \cdot f(Z_i). \label{eq:omega_def}
        \end{equation}
        \State \indent \textbf{Define } $\cL\ind{j}: \cF \to \rr$ such that \label{line:fbar_ftpl}
        \begin{equation}
            \cL^{(j)}(f) = \sum_{(X_i, Y_i) \in \cD } \ell(f(X_i), Y_i) + \eta \cdot \omega_m\ind{j}(f).
        \end{equation}
        \State \indent \textbf{Define } $\fbar_j = \erm(\cL^{(j)}, \cF)$ \label{line:fbar_j}
    }
    \State \textbf{Define} $\fbar = \frac 1J \cdot \sum_{j = 1}^J \fbar_j$. \label{line:averaging}
    \State \textbf{Output } $\fhat = \perturb(\fbar, \cQ, \gamma, \cDtil)$ \Comment{By running Algorithm \ref{alg:output_perturb}}
    \end{algorithmic}
    \caption{Oracle Efficient Private Learner (Perturbation)}
    \label{alg:general_alg}
 \end{algorithm}
 \begin{theorem}\label{thm:general_ftpl_ub}
    Suppose that $\cF : \cX \to [-1,1]$ is a function class and $\mu \in \Delta(\cX)$ is a measure such that $\inf_{f \in \cF} \norm{f}_{\mu} \geq \frac 23$.  Let $\ell: [-1,1] \times [-1,1] \to [0,1]$ be a loss function that is convex and $\lambda$-Lipschitz in its first argument.  If $\cQ = \cN(0,1)$ in \Cref{alg:output_perturb}, then for any $\epsilon, \delta, \alpha, \beta \in (0,1)$, there are choices of $\eta, \gamma > 0$ and $J, m \in \bbN$, all polynomial in problem parameters, such that if 
    \begin{align}
        n = \emph{\poly}\left(\sup_m \cGbar_m(\cF), \epsilon^{-1}, \alpha^{-1}, \log\left(  \frac 1\delta \right), \log\left( \frac 1\beta \right), \lambda \right),    
    \end{align}
    then the $\fhat$ returned by \Cref{alg:general_alg} is $(\epsilon, \delta)$-differentially private.  If $\nu_x$ is $\sigma$-smooth with respect to $\mu$, then $\fhat$ is an $(\alpha, \beta)$-learner with respect to $\nu_x$ and $\ell$.
 \end{theorem}
We emphasize that \Cref{alg:general_alg} is \emph{always differentially private}, independent of $\nu$; however, our algoirthm is only a good learner if $\nu_x$ is smooth with respect to $\mu$.  We remark that all of the conditions in \Cref{thm:general_ftpl_ub} are standard with the exception of the assumption that $\norm{f}_{\mu} \geq \frac 23$ for all $f \in \cF$.  This condition is easy to ensure by setting $\mutil = \frac {\mu + 2 \cdot \delta_{\zstar}}{3}$, where $\zstar$ is a distinguished point such that $f(\zstar) = 1$ for all $f \in \cF$; note that this process deflates $\sigma$ at most by a factor of 3 while ensuring the lower bound on the norm of $f$.  Replacing $\mu$ by $\mutil$ then suffices to ensure that \Cref{thm:general_ftpl_ub} holds.

We further remark that it is classical that the complexity notion $\sup_m \cGbar_m(\cF)$ is upper bounded by $\sqrt{\vc(\cF)}$ for binary function classes and $\sqrt{\log(\abs{\cF})}$ for finite classes \citep{wainwright2019high}, ensuring that the proven sample complexity is polynomial in all standard notions of function class complexity.  For even more complex function classes, where $\cGbar_m(\cF) = \omega(1)$, similar results hold, although with worse rates; further dicussion, as well as the precise polynomial dependence of hyperparameters and sample complexity, can be found in \Cref{app:proofs}.

While \Cref{alg:general_alg} succeeds in our desiderata under general assumptions, the sample complexity is a large polynomial of the desired accuracy.  Indeed, in \Cref{thm:general_full} found in \Cref{app:concluding_proofs}, we see that the sample complexity of \Cref{alg:general_alg} scales only like $O\left( \vc(\cF) \cdot \epsilon^{-3} \cdot \alpha^{-14} \right)$, which is significantly worse than the $O\left( \vc(\cF) \cdot \alpha^{-2} \right)$ sample complexity that a non-private algorithm such as ERM can achieve \citep{wainwright2019high}.  Furthermore, we are unable to achieve a pure differential privacy guarantee with this algorithm.  We now address both issues by providing an improved algorithm in the special case that the function class $\cF$ is \emph{convex}.  While we still use \Cref{alg:output_perturb} as a subroutine, in \Cref{alg:ftrl_alg}, motivated by the difference between Follow the Perturbed Leader (FTPL) and Follow the Regularized Leader (FTRL) \citep{kalai2005efficient,cesa2006prediction} in online learning, we modify the way in which we choose our initial estimator $\fbar$.  In particular, we eliminate the averaging step and redefine $\omega$ to be a strongly convex \emph{regularizer} instead of a Gaussian Process \emph{perturbation}.  More specifically, we define $\cL$ in \eqref{eq:reg_def} as the empirical loss regularized by $\norm{\cdot}_m^2$ and output $\fbar = \erm(\cL, \cF)$.
\begin{algorithm}
    \begin{algorithmic}[1]
    \State \textbf{Input } Oracle $\erm$, perturbation parameter $\eta > 0$, public data set $\cDtil = \left\{ Z_1, \dots, Z_m \right\}$, private data set $\cD = \left\{ (X_i, Y_i) | 1 \leq i \leq n \right\}$, function class $\cF$, loss function $\ell$, noise level $\gamma > 0$, number of iterations $J \in \bbN$, noise distribution $\cQ \in \Delta(\rr)$.
    \State \textbf{Define } $\cL: \cF \to \rr$ such that
    \begin{equation}\label{eq:reg_def}
        \cL(f) = \sum_{(X_i, Y_i) \in \cD } \ell(f(X_i), Y_i) + \eta \cdot \norm{f}_m^2.
    \end{equation}
    \State \textbf{Define } $\fbar = \erm(\cL, \cF)$ \label{line:fbar_ftrl}
    \State \textbf{Output } $\fhat = \perturb(\fbar, \cQ, \gamma, \cDtil)$ \Comment{By running Algorithm~\ref{alg:output_perturb}}
    \end{algorithmic}
    \caption{Oracle Efficient Private Learner (Regularization)}
    \label{alg:ftrl_alg}
 \end{algorithm}
We have the following result:
\begin{theorem}\label{thm:cvx_ftrl}
    Suppose that $\cF: \cX \to [-1,1]$ is a convex function class and $\ell:[-1,1] \times [-1,1] \to [0,1]$ is convex and $\lambda$-Lipschitz in its first argument.  Suppose that $Z_1, \dots, Z_m \sim \mu$ are independent and $\cQ = \cN(0,1)$.  Then there are $\eta, \gamma, m$ polynomial in problem parameters such that, if
    \begin{align}
        n = \Omegatil\left( \emph{\poly}\left(\log\left( \beta^{-1} \right), \log\left( \delta^{-1} \right), \lambda   \right)\cdot (\sup_m \cGbar_m(\cF))^2 \cdot \epsilon^{-1} \cdot \alpha^{-5} \right)
    \end{align}
    then the $\fhat$ returned by \Cref{alg:ftrl_alg} is $(\epsilon, \delta)$-differentially private.  If $\nu_x$ is $\sigma$-smooth with respect to $\mu$, then $\fhat$ is an $(\alpha, \beta)$-learner with respect to $\nu_x$ and $\ell$.  Furthermore, if $\cQ = \Lap(1)$, and
    \begin{align}
        n = \emph{\poly}\left(\sup_m \cGbar_m(\cF), \epsilon^{-1}, \alpha^{-1}, \log\left( \beta^{-1} \right), \lambda\right),
    \end{align}
    then \Cref{alg:ftrl_alg} is $\epsilon$-\emph{purely} differentially private and is an $(\alpha, \beta)$-PAC learner for any $\sigma$-smooth $\nu_x$.
\end{theorem}
As in the case of \Cref{thm:general_ftpl_ub}, we can easily generalize \Cref{thm:cvx_ftrl} to apply to function classes $\cF$ where $\cGbar_m(\cF) = \omega(1)$ at the cost of worse polynomial dependence in the sample complexity.  We again omit this case for the sake of simplicity.  While the sample complexity of \Cref{alg:ftrl_alg} is a marked improvement over that of \Cref{alg:general_alg}, it remains a far cry from the desired $O\left( \alpha^{-2} \right)$ rates of non-private learning that computationally \emph{inefficient} private algorithms leveraging public data are able to achieve \citep{BassilyMA19}; we leave the interesting question of producing an oracle-efficient private algorithm with optimal sample complexity to future work.

Finally, we remark that even in the case where $\cF$ is not convex, \Cref{alg:ftrl_alg} can be applied to $\conv
(\cF)$, the convex hull of $\cF$, if we assume the learner has access to $\erm'$, a stronger ERM oracle that can optimize over $\conv(\cF)$.  In this case, \Cref{thm:cvx_ftrl} supercedes \Cref{thm:general_ftpl_ub} as it is easy to see that $\cGbar_m(\cF) = \cGbar_m(\conv(\cF))$ and thus the sample complexity of \Cref{alg:ftrl_alg} is strictly better than that of \Cref{alg:general_alg} and the pure differential privacy result applies.  We now turn to the proofs of our main results.

%% file: body/analysis.tex
\section{Analysis Techniques}\label{sec:analysis}
In this section, we outline the proofs of our main results, with full details and technical lemmata deferred to \Cref{app:proofs}.  As is suggested by our template, the proof of the privacy part of \Cref{thm:general_ftpl_ub} rests on two results: the first shows that if $\cQ$ is a standard Gaussian (resp. exponential) then stability of $\fbar$ with respect to $\norm{\cdot}_m$ can be translated into differential privacy.  The second shows that $\fbar$ will be stable with respect to $\norm{\cdot}_m$.  Similarly, the proof that $\fhat$ is a good learner first shows that $\fbar$ is a good learner and then that $\fhat$ and $\fbar$ are close.  We begin with the more technically novel parts and show that, under standard assumptions, \Cref{alg:general_alg,alg:ftrl_alg} result in $\fbar$ that are stable in $\norm{\cdot}_m$.  In our proof of \Cref{thm:general_ftpl_ub}, we provide an improved analysis of the Gaussian anti-concentration result from \citet{block2022smoothed}, which may be of independent interest.  We prove the stability of \Cref{alg:ftrl_alg} using a technique common in online learning.  We then proceed with the more standard analysis and show that stability in $\norm{\cdot}_m$ can be boosted to a differential privacy guarantee using the Gaussian and Laplace Mechanisms \citep{DworkMNS06}.  Finally, we apply standard learning theoretic techniques to show that $\fhat$ is a good learner.  

\subsection{Stability Analysis}
In this section, we explain how to prove that \Cref{alg:general_alg,alg:ftrl_alg} are stable with respect to $\norm{\cdot}_m$.  Our stability results further cement the connections between differential privacy and online learning noted in \citet{abernethy2019online} as both algorithms are primarily motivated by online learning techniques.  We begin by describing the stability analysis of \Cref{alg:general_alg}.

The key lemma underlying the stability of \Cref{alg:general_alg} is an improved version of a Gaussian anti-concentration result from \citet{block2022smoothed}, which may be of independent interest. The abstract anti-concentration lemma is the following:
\begin{proposition}\label{prop:banach_gaussian_anti_concentration}
    Let $\cF$ denote a (separable) subspace of the unit ball with respect to a norm $\norm{\cdot}$ induced by an inner product $\inprod{\cdot}{\cdot}$ and let $m,m': \cF \to \rr$ denote measurable functions such that $\sup_{f \in \cF} \abs{m(f) - m(f')} \leq \tau$.  If $\omega$ is a centred Gaussian process on $\cF$ with covariance kernel given by $\inprod{\cdot}{\cdot}$, $\Omega(f) = m(f) + \eta \cdot \omega(f)$, $\fbar = \argmin_{f \in \cF} \Omega(f)$, and $\Omega'$ and $\fbar'$ are defined similarly, then for any $\rho, \tau > 0$, it holds that
    \begin{align}
        \pp\left( \norm{\fbar - \fbar'} > \rho \right) \leq \frac{8 \tau}{\rho^4 \kappa^2 \eta} \cdot \ee\left[ \sup_{f \in \cF} \omega(f) \right],
    \end{align}
    where $\kappa^2 = \inf_{f \in \cF} \ee\left[ \omega(f)^2 \right]$.
\end{proposition}
The proof proceeds in a similar way to that of Lemma 33 from \citet{block2022smoothed}, but involves a tighter analysis in several steps in order to improve the bound.  The intuition for the result is straightforward: if $\fbar$ is the minimizer of the Gaussian process $\Omega$, then with reasonable probability, \emph{almost minimizers} of $\Omega$ (as measured by the tolerance $\tau$) are within a radius $\rho$ of $\fbar$ as long as the Gaussian process is nontrivial in the sense that all indices $f$ have sufficiently high variance.  Moreover, the quantitative control on the probability of this event depends in a natural way both on $\tau$ and $\rho$ as well as on the Gaussian process $\omega$: more complex spaces $\cF$ and lower variance processes lead to a worse anti-concentration guarantee.  Finally, we note that \Cref{prop:banach_gaussian_anti_concentration} is an improvement of Lemma 33 from \citet{block2022smoothed} in that the quantitative bound on the probability of anti-concentration is tighter by polynomial factors in $\rho, \eta$, and $\kappa$.

Like essentially all anti-concentration results \citep{chernozhukov2015comparison}, \Cref{prop:banach_gaussian_anti_concentration} holds only with moderate probability in the sense that the guarantee is polynomial in the scale $\rho$; this fact is in contradistinction to \emph{concentration} inequalities which tend to hold with high probability exponential in the scale.  This discrepancy is precisely what motivates the averaging in Line \ref{line:averaging} of \Cref{alg:general_alg}.  Indeed, we can use \Cref{prop:banach_gaussian_anti_concentration} to show that if $\fbar_j$ is as in Line \ref{line:fbar_j} of \Cref{alg:general_alg} and $\fbar_j'$ is defined analogously with respect to $\cD'$, then with moderate probability $\norm{\fbar_j - \fbar_j'}_m$ is small.  Using Jensen's inequality and a standard chernoff bound, we can then boost this moderate probability guarantee into a high probability guarantee to show that if $J$ is sufficiently large, then $\norm{\fbar -\fbar'}_m$ is small with high probability.  We formalize this argument in the following lemma:
\begin{lemma}[Stability of \Cref{alg:general_alg}]\label{lem:anticoncentration_stability}
    Suppose that $\cF: \cX \to [-1,1]$ is a function class and $\ell: [-1,1]^{\times 2} \to [0,1]$ is a bounded loss function.  Suppose that $\cD, \cD'$ are neighboring datasets and let $\fbar$ be as in Line \ref{line:fbar_ftpl} of \Cref{alg:general_alg} and $\fbar'$ be defined analogously with respect to $\cD'$.  Then for any $\rho, \delta > 0$, with probability at least $1 - \delta$, over the Gaussian processes $\omega\ind{j}$,
    \begin{align}
        \norm{\fbar - \fbar'}_m \leq \frac{2}{(\eta \cdot n)^{1/3} \kappa^{2/3}} \cdot \left( \ee\left[ \sup_{f \in \cF} \omega_m(f) \right] \right)^{1/3} + \sqrt{\frac{\log\left( \frac 1\delta \right)}{J}}.
    \end{align}
\end{lemma}
We note that the worse dependence on $\eta$ in \Cref{lem:anticoncentration_stability} as compared to \Cref{prop:banach_gaussian_anti_concentration} arises from integrating the tail bound to obtain the control on $\norm{\fbar_j - \fbar_j'}_m$ in expectation necessary to apply Jensen's inequality; details can be found in \Cref{app:stability_analysis}.

We now turn to the stability of \Cref{alg:ftrl_alg}.  The proof is based on a technique borrowed from online learning and the analysis of the \emph{Follow the Regularized Leader} (FTRL) algorithm \citep{gordon1999regret,cesa2006prediction}.  
\begin{lemma}[Stability of \Cref{alg:ftrl_alg}]\label{lem:ftrl_stability}
    Suppose that $\ell$ is convex and $\lambda$-Lipschitz in its first argument.  Let $\cD, \cD'$ denote neighboring data sets and let $\fbar$ denote the output of Line \ref{line:fbar_ftrl} in \Cref{alg:ftrl_alg} and $\fbar'$ be the analogous output evaluated on $\cD'$.  If $\cF$ is convex, then, \iftoggle{colt}{$\norm{\fbar - \fbar'}_{m} \leq \frac{2}{\sqrt{\eta \cdot n}}$.}{
    \begin{align}
        \norm{\fbar - \fbar'}_{m} \leq \frac{2}{\sqrt{\eta \cdot n}}.
    \end{align}
    }
\end{lemma}
The proof of \Cref{lem:ftrl_stability} can be found in \Cref{app:ftrl_stability_analysis} and rests on elementary properties of strongly convex functions.  We note that relative to \Cref{lem:anticoncentration_stability}, the dependence on $\eta$ in \Cref{lem:ftrl_stability} is improved, which in turn leads to the better sample complexity exhibited in \Cref{thm:cvx_ftrl} relative to that in \Cref{thm:general_ftpl_ub}.  With stability of \Cref{alg:general_alg,alg:ftrl_alg} thus established, we proceed to analyze the effect of the output perturbation.

\subsection{Output Perturbation Analysis}\label{subsec:output_perturb_analysis}
We now turn to the analysis of \Cref{alg:output_perturb}.  In order to boost a stability-in-norm guarantee into one for differential privacy while remaining a good learner, we require the output perturbation to be sufficiently small as to not not affect the learning guarantee of $\fbar$ while at the same time being sufficiently large as to ensure privacy.  We balance these two competing objectives by tuning the variance of the added noise.  This part of the analysis is relatively standard in the differential privacy literature \citep{chaudhuri2011differentially,neel2019use}, with the bound on the size of the output perturbation following from standard tail bounds on Gaussian and Laplace random vectors.  The privacy guarantees are similarly standard and summarized in the following lemma:
\begin{lemma}\label{lem:alg_output_perturb}
    Suppose that $\fbar \in \cF$ is the output of some algorithm $\cA: \cD \to \cF$ that is $\rho$-stable with respect to $\norm{\cdot}_m$, i.e., for any neighboring data set $\cD'$, it holds that $\norm{\cA(\cD) - \cA(\cD')}_m \leq \rho$.  Then applying \Cref{alg:output_perturb} with $\cQ = \cN(0,1)$ to $\fbar$ results in an $(\epsilon, \delta)$-private algorithm if \iftoggle{colt}{$\frac{m}{2\gamma^2} \left( 1 + \gamma \cdot \sqrt{\log\left( \frac 1\delta \right)} \right) \rho \leq \epsilon$.}{
    \begin{align}
        \frac{m}{2\gamma^2} \left( 1 + \gamma \cdot \sqrt{\log\left( \frac 1\delta \right)} \right) \rho \leq \epsilon.
    \end{align} 
    }
    Similarly, if $\cQ = \Lap(\gamma)$, then the algorithm is $\epsilon$-purely private if \iftoggle{colt}{$ m^{3/2} / \gamma \cdot \rho \leq \epsilon$.}{
    \begin{align}
        \frac{m^{3/2}}{\gamma} \cdot \rho \leq \epsilon.
    \end{align}
    }
\end{lemma}
This standard result is proved in \Cref{app:output_perturb_analysis}.  The balance between privacy and learning is quantified in the choices of $m$ and $\gamma$.  If $\gamma$ is too large, then $\fbar$ will be private but a poor learner, whereas the opposite occurs if $\gamma$ is too small.  Similarly, if $m$ is too large then privacy is reduced whereas if $m$ is too small then $\norm{\cdot}_m$ is a poor approximation for $\norm{\cdot}_{\mu}$.  We now describe how to conclude the proofs of \Cref{thm:general_ftpl_ub,thm:cvx_ftrl}.

\subsection{Learning Guarantees and Concluding the Proof}\label{subsec:pac}
By combining \Cref{lem:anticoncentration_stability}  (resp. \Cref{lem:ftrl_stability}) with \Cref{lem:perturb_boosting}, we can establish the privacy of \Cref{alg:general_alg} (resp. \Cref{alg:ftrl_alg}) as long as the tuning parameters $m, \gamma, \eta$, and $J$ are chosen correctly.  We now sketch the proof that these algorithms comprise good learners in the sense of \Cref{def:pac}.  

We break our proof into three components, the first two of which are standard learning theoretic results.  The first lemma says that if $m \gg 1$, then $\norm{\cdot}_m$ is a good approximation for $\norm{\cdot}_\mu$: 
\begin{lemma}\label{lem:norm_comparison_informal}
    Let $\cF: \cX \to [-1,1]$ be a bounded function class and let $Z_1, \dots, Z_m \sim \mu$ be independent samples.  Then for any $\beta > 0$ it holds with probability at least $1 - \beta$ that for all $f \in \cF$,
    \begin{align}
        \norm{f}_\mu \leq 2 \cdot \norm{f}_m + \widetilde{O}\left( \frac{\cGbar_m(\cF) + \sqrt{\log\left( 1/\beta \right)}}{\sqrt{m}} \right).
    \end{align}    
\end{lemma}
\Cref{lem:norm_comparison_informal} is a standard bound from learning theory \citep{bousquet2002concentration,rakhlin2017empirical} and is proved in \Cref{app:norm_comparison} for the sake of completeness.  The second component is given by \Cref{lem:uniform_deviations} in \Cref{app:learning_theory}, which amounts to a classical uniform deviations bound for the empirical process, ensuring that if $n \gg 1$, then $L_{\cD}(f) \approx L(f)$ for all $f \in \cF$.  The final step is the following simple lemma, which ensures that if $\eta$ is not too large, then $L_{\cD}(\fbar) \approx L_{\cD}(\ferm)$:
\begin{lemma}\label{lem:small_regularizer}
    Let $\cF: \cX \to [-1,1]$ be a bounded function class and let $R: \cF \to \rr$ be an arbitrary, possibly random, regularizer.  Let \iftoggle{colt}{$\ferm \in \argmin_{f \in \cF} L_{\cD}(f)$ and $\fbar \in \argmin_{f \in \cF} L_{\cD}(f) + R(f)$.}{
    \begin{align}
        \ferm \in \argmin_{f \in \cF} L_{\cD}(f) \qquad \text{and} \qquad \fbar \in \argmin_{f \in \cF} L_{\cD}(f) + R(f).
    \end{align}
    }
    Then,
    \begin{align}
        L_{\cD}(\fbar) \leq L_{\cD}(\ferm) + \sup_{f, f' \in \cF} R(f) - R(f').
    \end{align}
\end{lemma}
\Cref{lem:small_regularizer} is a simple computation proved in \Cref{app:learning_theory}.  Letting $R(f)$ be either $\eta \cdot \omega_m\ind{j}(f)$ in \Cref{alg:general_alg} or $\eta \cdot \norm{f}_m^2$ in \Cref{alg:ftrl_alg} demonstrates that if $\eta$ is not too large, then $\fbar$ performs similarly to $\ferm$.  To prove that \Cref{alg:general_alg,alg:ftrl_alg} produce good learning algorithms then, it suffices to combine these three components, observing first that $L(\fbar)$ is close to optimal if $n \gg 1$ and $\eta$ is not too large, second that $\norm{\fbar - \fhat}_\mu \ll 1$ if $m \gg 1$ and $\gamma$ is sufficently small, and third that $\abs{L(\fhat) - L(\fbar)} \lesssim \norm{\fbar - \fhat}_\mu$ if $\nu_x$ is $\sigma$-smooth with respect to $\mu$ and $\ell$ is $\lambda$-Lipschitz in its first argument.  Combining these results concludes the proofs of \Cref{thm:general_ftpl_ub,thm:cvx_ftrl}.  A detailed and rigorous argument for both proofs is presented in \Cref{app:proofs}.

As a final remark, we note that in the case of \Cref{alg:general_alg}, convexity of $\ell$ in the first argument is irrelevant to the privacy guarantee despite being necessary for learning.  Indeed, for $\fbar$ returned by Line \ref{line:fbar_ftpl} in \Cref{alg:general_alg} to be proven a good learner, we apply Jensen's and the above argument that ensures that $\fbar_j$ is a good learner.  Interestingly, on the other hand, convexity in $\ell$ is irrelevant to the learning guarantee of \Cref{alg:ftrl_alg} while it is essential to the privacy guarantee.  Further understanding the role that such structural assumptions play in allowing privacy is an interesting direction for future work.

%% file: body/main-classification.tex
\newcommand{\WD}{\mathrm{WD}}
\newcommand{\werm}{\mathcal{O}^*}
\newcommand{\indi}{\mathbbm{1}}
\newcommand{\proj}[2]{{#1}\vert_{#2}}
\newcommand{\projected}[1]{{#1}_{\tilde{\cD}}}
\newcommand{\sgn}{\mathrm{sgn}}
\newcommand{\ww}{\boldsymbol{\xi}}
\newcommand{\map}{\Psi}
\newcommand{\Uni}{\mathrm{Uni}}
\newcommand{\rrspm}{\mathsf{RRSPM}}
\newcommand{\rspm}{\mathsf{RSPM}}
\newcommand{\lap}{\mathrm{Lap}}

\section{Differentially Private Classification}
In the previous section, we presented a private algorithm for general, real-valued loss functions. Here, we turn to the special case of classification, where we provide an algorithm with improved rates.
Formally, binary classification is a special case of Definition~\ref{def:pac}, where $\cF:\cX\to\{0,1\}$ and $\ell$ is the indicator loss. 

Much like \Cref{alg:general_alg,alg:ftrl_alg}, our approach to classification in \Cref{alg:RRSPM} relies on minimizing a perturbed empirical loss over $\cF$ and projecting the output $\ftil$ onto the public data.  Unlike in these earlier algorithms, which require a further perturbation of the output in order to boost stability into differential privacy, in the special case of classification we are able to circumvent this second perturbation and return any $\fhat$ that agrees with $\ftil$ on the public data.  This is accomplished by carefully choosing the initial perturbation to the ERM objective (see \eqref{eq:objective_perturbation_rrspm}) so that the predictions of $\ftil$ on the public data satisfy differential privacy without ensuring some form of stability in norm. As a result, our improved rates then follow from lack of a second perturbation.  We present the following guarantee for our classification algorithm, whose pseudo-code can be found in \Cref{alg:RRSPM}.

\begin{algorithm}
    \begin{algorithmic}
        \State \textbf{Input } ERM oracle $\erm$, dataset $\cD=\{\left(X_i,Y_i\right)\mid 1\leq i\leq n\}$, hypothesis class $\cF$, smoothness parameter $\sigma$, loss function $\ell:\cY\times\cY\to \{0,1\}$, arbitrary $\cQ\in\Delta(\rr)$.
        \State \textbf{Draw } $\tilde{\cD}=(\tilde{\cD}_x,\tilde{\cD}_y)$ where $\tilde{\cD}_x=\{Z_1,\dots ,Z_m\}$ and $\tilde{\cD}_y=\{\tilde{Y}_1,\dots ,\tilde{Y}_m\}$ such that $Z_i\sim \mu$ and $\tilde{Y}_i\sim\Uni(\{0,1\})$, for all $i\in[m]$.
        \State \textbf{Draw } weights $\ww=\{\xi_1,\dots,\xi_m\}$ such that $\xi_i\sim \lap(2m/\varepsilon)$.
        \State \textbf{Define } $\cL_{\ww,\cD,\tilde{\cD}}: \cF \to \rr$ such that
            \begin{align}\label{eq:objective_perturbation_rrspm}
                \cL_{\ww,\cD,\tilde{\cD}}(f) = \sum_{i=1}^{n} \ell(f(X_i), Y_i) + \sum_{i=1}^m\xi_i \cdot \ell(f(Z_i),\tilde{Y}_i).
            \end{align}
        \State \textbf{Get} $\tilde{f}=\erm(\cF,\cL_{\ww,\cD,\tilde{\cD}})$.
        \State \textbf{Output} $\hat{f}=\perturb(\tilde{f},\cQ,\gamma=0,\tilde{D}_x)$
        \Comment{By running Algorithm \ref{alg:output_perturb}}
    \end{algorithmic}
    \caption{Rounded Report Separator Perturbed Minimum Algorithm (RRSPM)}
    \label{alg:RRSPM}
\end{algorithm}

\begin{theorem} \label{thm:classification-upper-bound}
    Suppose that $\cF:\cX\to\cY$ is a function class of VC dimension $d$ and $\ell:\cY\times\cY\to\{0,1\}$ is the indicator loss. Suppose that $Z_1,\dots ,Z_m\sim\mu$ and $\xi_1,\dots ,\xi_m\sim\lap(2m/\varepsilon)$ are independent. Then there is a choice of $m$ polynomial in the problem parameters such that if
    \begin{align}
        n=\tilde{\Omega}(d^2\varepsilon^{-1}\alpha^{-5}\log(\beta^{-1})),
    \end{align}
    then the $\hat{f}$ returned by Algorithm~\ref{alg:RRSPM} is $\varepsilon$-pure differentially private. If $\nu_x$ is $\sigma$-smooth with respect to $\mu$, then $\hat{f}$ is an $(\alpha,\beta)$-learner with respect to $\nu_x$ and $\ell$. Furthermore, for some $C > 0$, if $\xi_1,\dots ,\xi_m\sim\cN(0,C\sqrt{m\log(1/\delta)}/\varepsilon)$ then there is a choice of $m$ polynomial in the problem parameters such that if 
    \begin{align}
        n=\tilde{\Omega}(d^2\varepsilon^{-1}\alpha^{-4}\log^{1/2}(1/\delta)\log(\beta^{-1})),
    \end{align}
    then the $\hat{f}$ returned by Algorithm~\ref{alg:RRSPM} with Gaussian perturbations is $(\varepsilon,\delta)$-differentially private and is an $(\alpha, \beta)$-PAC learner with respect to any $\sigma$-smooth $\nu_x$.
\end{theorem}
\begin{remark}
    Note that the sample complexity we get in the above theorems is in the general, agnostic setting. In the realizable setting, where some $\fstar \in \cF$ perfectly predicts the $Y$ from the $X$, we get a sample complexity of $n=\tilde{\Omega}(d^2\varepsilon^{-1}\alpha^{-3}\log(\beta^{-1}))$ for $\varepsilon$-pure differential privacy and $n=\tilde{\Omega}(d^2\varepsilon^{-1}\alpha^{-2.5}\log^{1/2}(1/\delta)\log(\beta^{-1}))$ for $(\varepsilon,\delta)$-differential privacy.
\end{remark}
We emphasize that \Cref{thm:classification-upper-bound} attains the improved $O(\alpha^{-5})$ sample complexity (even $O(\alpha^{-4}$ for approximate differential privacy)), which is significantly better than the $O(\alpha^{-14})$ from \Cref{thm:general_ftpl_ub}.
While this is a major improvement, it still falls short of the desired $O(\alpha^{-2})$ statistical rates achievable by inefficient algorithms from \citep{BassilyMA19}. We leave the interesting question of whether improved sample complexity is possible to future work. We now briefly sketch the proof of \Cref{thm:classification-upper-bound}. 

\subsection{Privacy Analysis of Algorithm \ref{alg:RRSPM}}

While the privacy of \Cref{alg:general_alg,alg:ftrl_alg} is proven in two steps, by first demonstrating stability and then leveraging the output perturbation to ensure privacy, the privacy of \Cref{alg:RRSPM} is proven directly.  Our approach is motivated by techniques from \citet{neel2019use}, which adapt the earlier notion of \emph{separator sets} from \citet{goldman1993exact,syrgkanis2016efficient,dudik2020oracle} to the setting of differential privacy.  Unlike those works, however, we do not require the strong assumption that $\cF$ has a small separator set and our results hold for general VC function classes.  The main technical result that ensures privacy of \Cref{alg:RRSPM} demonstrates that the \emph{projection} $\cF$ to the public data set $\tilde{\cD}_x$ is private with respect to $\cD$, where we let $\proj{\cF}{\tilde{\cD}_x} = \left\{ (f(Z_i))_{1\leq i \leq m} | f \in \cF \right\}$.  We have the following privacy guarantee for $\ftil(\tilde{D}_x) \in \proj{\cF}{\tilde{\cD}_x}$:

\begin{lemma}(Privacy over Projection)\label{lem:projpriv_body}
Let $\cD,\cD'$ be arbitrary datasets containing $n$ points each. Let $\tilde{\cD}_x$ be a set of $m$ points $Z_1,\dots ,Z_m \in \cX$. Let $\tilde{Y}_1,\dots ,\tilde{Y}_m\in\{0,1\}$ be the set of corresponding labels. Then for all measurable $\cH\subseteq\proj{\cF}{\tilde{\cD}_x}$
\begin{align}
    \pp(\erm(\proj{\cF}{\tilde{\cD}_x},\cL_{\ww,\cD,\tilde{\cD}})\in\cH)\leq e^{\varepsilon} \cdot \pp(\erm(\proj{\cF}{\tilde{\cD}_x},\cL_{\ww,\cD',\tilde{\cD}})\in\cH),
\end{align}
where $\cL_{\ww,\cD,\tilde{\cD}}$ is defined as in \eqref{eq:objective_perturbation_rrspm}.
\end{lemma}
With the above lemma in hand, the privacy of \Cref{alg:RRSPM} follows immediately from the post-processing property of differential privacy.  We provide a full proof of \Cref{lem:projpriv_body} in \Cref{app:classification} and now turn to the accuracy guarantee.

\subsection{Concluding the Proof of Theorem \ref{thm:classification-upper-bound}}
Proving that \Cref{alg:RRSPM} is an $(\alpha,\beta)$-learner whenever $\nu_x$ is $\sigma$-smooth with respect to $\mu$ is similar to the approach taken in \Cref{subsec:pac}, with the critical difference that in the absence of the second perturbation, we are able to achieve a stronger guarantee on the difference between $\ftil$ and $\fhat$.  Indeed, much as in the previous analysis, we observe that as $m$ increases, the suboptimality of the intermediate $\ftil$ is driven up, while the difference between $\ftil$ and $\fhat$ is driven down, thereby requiring a careful balance; here, however, we do not also need to account for the balancing of the variance $\gamma$.  We provide a full proof of the accuracy guarantee in \Cref{app:classification}.

%% file: body/acks.tex
AB acknowledges support from the National Science Foundation Graduate Research Fellowship under Grant No.1122374 as well as the Simons Foundation and the National Science Foundation through awards DMS-2031883 and DMS-1953181.  MB acknowledges support from the National Science Foundation through award NSF CNS-2046425 and from a Sloan Research Fellowship. RD acknowledges support from the National Science Foundation through award NSF CNS-2046425. AS acknowledges support from the Apple AI+ML fellowship.  AB also would like to thank Satyen Kale and Claudio Gentile for helpful discussions.

%% file: body/app_gaussian_anticoncentration.tex
\section{Gaussian Anti-Concentration and Proof of Lemma \ref{prop:banach_gaussian_anti_concentration}}\label{app:gaussian_anticoncentration}
In this section we present and prove a more general version of \Cref{prop:banach_gaussian_anti_concentration}.  We begin by defining a Gaussian process and then state and prove the result.  We then show how \Cref{prop:banach_gaussian_anti_concentration} follows as an immediate corollary.  To begin, we recall the formal definition of a Gaussian process.
\begin{definition}\label{def:gp_formal}
    Let $T$ be an index set and $m: T \to \rr$ be a function.  Let $K: T \times T \to \rr$ be a covariance kernel in the sense that for any $t_1, \dots, t_n \in T$, the matrix $\left( K(t_i, t_j) \right)_{i,j \in [n]}$ is positive semi-definite.  We say that $\omega: T \to \rr$ is a Gaussian process with mean function $m$ and covariance kernel $K$ if for any $t_1, \dots, t_n \in T$, the random vector $\left( \omega(t_i) \right)_{i \in [n]}$ is Gaussian with mean $\left( m(t_i) \right)_{i \in [n]}$ and covariance matrix $\left( K(t_i, t_j) \right)_{i,j \in [n]}$.  We say that $\omega$ is a centered Gaussian process if $m$ is identically zero.
\end{definition}
Note that by \citet[Theorem 1.11]{le2016brownian} such a process always exists given $m, K$.  Furthermore, we note that $K$ induces a semi-metric $d$ on $T$ by letting $d(t,t')^2 = \ee\left[ (\omega(t) - \omega(t'))^2 \right]$.  We now prove the following result, which is a tighter version of \citet[Lemma 33]{block2022smoothed}.
\begin{theorem}[Gaussian Anti-concentration]\label{thm:gaussiananticoncentration}
    Let $T$ be a set, $m: T \to \rr$ be a mean function and $K: T \times T \to \rr$ be a covariance kernel (in the sense of being positive definite).  Let $d$ denote the metric induced by $K$ and suppose that $m$ is continuous  with respect to $d$, and the metric space $(T, d)$ is separable and compact.  Let $\omega$ denote a Gaussian process on $T$ with covariance $K$ and for $\eta > 0$, let
    \begin{align}
        \Omega(t) = m(t) + \eta \cdot \omega(t)
    \end{align}
    be an offset Gaussian process.  We further suppose that $\omega$ is taken to be a version with almost surely continuous paths $t \mapsto \omega(t)$ and that $0 < \kappa \leq K(t,t) \leq 1$ for all $t \in T$.  Let
    \begin{equation}
        \tstar = \argmin_{t \in T} \Omega(t),
    \end{equation}
    and,
    \begin{equation}
        \cE(\rho, \tau) = \left\{\text{there exists } s \in T \text{ such that } \frac{K(s,\tstar)}{K(\tstar, \tstar)} \leq 1 - \rho^2 \text{ and } \Omega(s) \leq \Omega(\tstar) + \tau \right\},
    \end{equation}
    for $\rho, \tau > 0$.  The following holds:
    \begin{align}
        \pp\left(\cE(\rho, \tau)\right) \leq  \frac{\tau}{\rho^2 \eta \kappa^2} \cdot \ee\left[\sup_{t \in T} \omega(t)\right].
    \end{align}
\end{theorem}
Note that $\frac{K(s,t)}{K(t,t)}$ is a measure of how close $s$ and $t$ are to each other; indeed, in the special case where $\kappa = 1$ this is precisely the correlation and thus $s$ and $t$ are more closely related the closer this quantity is to 1.  Thus the event $\cE(\rho, \tau)$ can be interpreted to mean that there exists some point $s$ far from $\tstar$ (as governed by $\rho$) such that $\Omega(s)$ is almost minimal (as governed by $\tau$); in other words, \Cref{thm:gaussiananticoncentration} puts an upper bound on the probability that almost-minimizers of a Gaussian process lie far from the true minimizer.  We now prove \Cref{thm:gaussiananticoncentration}.
\begin{proof}[Proof of \Cref{thm:gaussiananticoncentration}]
    Note that by compactness of $T$ and almost sure continuity of $\Omega$, a minimizer of $\Omega$ exists almost surely; furthermore, by \cite[Lemma 2.6]{kim1990cube},  $\tstar$ is almost surely unique.  As $T$ is separable and $\Omega$ has almost surely continuous sample paths, it suffices to replace $T$ with a countable dense subset.  We will hereafter suppose without loss of generality that $T$ is countable.  For each $t \in T$, define the set
    \begin{align}
        A(t) = \left\{s \in T \bigg| \frac{K(s, t)}{K(t,t)} \leq 1 - \rho^2 \text{ and } \Omega(s) \leq \Omega(t) + \tau  \right\}.
    \end{align}
    It then suffices to lower bound the probability that $A(\tstar) = \emptyset$.  We compute
    \begin{align}
        \pp\left(\abs{A(\tstar)} = 0 \right) &= \sum_{t \in T} \pp\left(\tstar = t \text{ and } \abs{A(t)} = 0\right) \\
        &= \sum_{t \in T} \ee_y\left[\pp\left(\tstar = t \text{ and } \inf_{K(s, t) \leq (1 - \rho^2)K(t,t)} \Omega(s) \geq y + \tau| \Omega(t) = y\right) \right],
    \end{align}
    where the expectation is taken over the distribution of $\Omega(t)$.  Now, fix $t$ and let $\Omega_{t,y}$ denote the Gaussian process $\Omega$ conditioned on the event that $\Omega(t) = y$.  Let $m_{t,y}$ and  $\eta^2 \cdot K_t$ denote the mean and covariance processes of $\Omega_{t,y}$.  Critically, note that $K_t$ is independent of $y$ and for all $s \neq t$, we have
    \begin{align}
        m_{t,y}(s) = m(s) + \frac{K(s, t)}{K(t,t)}\left(y - m(t)\right).
    \end{align}
    Define the functions
    \begin{align}
        a(s) = \frac \tau{\rho^2} \cdot \frac{K(s, t) }{ K(t,t)}  \qquad \text{ and } \qquad b(s) = \frac{\tau}{\rho^2}  - a(s).
    \end{align}
    Now, note that if $K(s, t) \leq (1 - \rho^2)\cdot K(t,t)$, then
    \begin{align}\label{eq:conditionalineq1}
        b(s)  = \frac{\tau}{\rho^2}\left(1 - \frac{K(s, t)}{ K(t,t)}\right)  \geq \frac{\tau}{\rho^2 } \cdot \rho^2 = \tau.
    \end{align}
    We also have that $b(s) \geq 0$ for all $s$ by the fact that $K(s,t) \leq \sqrt{K(s,s) \cdot K(t,t)}$ and $K(s,s) \vee K(t,t) \leq 1$.  Furthermore, for all $s$, it holds that
    \begin{equation}\label{eq:conditionalineq2}
        m_{t, y + \frac{\tau}{\rho^2}}(s) = m_{t,y}(s) + a(s).
    \end{equation}
    Thus, for fixed $t \in T$ and $y \in \rr$, we have
    \begin{align}
        \pp&\left(\tstar = t \text{ and } \inf_{K(s, t) \leq (1 - \rho^2)K(t,t)} \Omega(s) \geq y + \tau| \Omega(t) = y\right) \\
        &\geq \pp\left(\tstar = t \text{ and } \inf_{K(s, t) \leq (1 - \rho^2)K(t,t)} \Omega(s) - b(s) \geq y| \Omega(t) = y\right) \\
        &= \pp\left(\tstar = t \text{ and } \inf_{K(s, t) \leq (1 - \rho^2)K(t,t)} \Omega(s) - b(s) - a(s) + a(s) \geq y| \Omega(t) = y\right) \\
        &= \pp\left(\tstar = t \text{ and } \inf_{K(s, t) \leq (1 - \rho^2)K(t,t)} \Omega(s) + a(s) \geq y + \frac{\tau}{\rho^2}| \Omega(t) = y\right) \\
        &= \pp\left(\tstar = t \text{ and } \inf_{K(s, t) \leq (1 - \rho^2)K(t,t)} \Omega(s) \geq y + \frac{\tau}{\rho^2}| \Omega(t) = y + \frac{\tau}{\rho^2}\right) \\
        &\geq \pp\left(\tstar = t \text{ and } \inf_{K(s, t) \leq (1 - \rho^2)K(t,t)} \Omega(s) \geq y + \frac{\tau}{\rho^2}| \Omega(t) = y + \frac{\tau}{\rho^2}\right),
    \end{align}
    where the first inequality follows from \eqref{eq:conditionalineq1}, the second equality follows from the construction, and the last equality follows from \eqref{eq:conditionalineq2} and the fact that $K_t$ is independent of $y$.  Now, denote
    \begin{equation}
        q_t(y) = (2 \pi K(t,t) )^{- \frac 12} \exp\left(- \frac{(y - m(t))^2}{2 \eta^2 K(t,t)}\right),
    \end{equation}
    the density of $\Omega(t)$ and note that we have
    \begin{align}
        \pp\left(\abs{A(\tstar)} = 0\right)&= \sum_{t \in T} \int_{-\infty}^\infty q_t(y) \pp\left(\tstar = t \text{ and } \inf_{K(s, t) \leq 1 - \rho^2} \Omega(s) \geq y + \tau| \Omega(t) = y\right) d y \\
        &\geq \sum_{t \in T} \int_{\infty}^\infty q_t(y)\pp\left(\tstar = t \text{ and } \inf_{K(s, t) \leq 1 - \rho^2} \Omega(s) \geq y + \frac{\tau}{\rho^2}| \Omega(t) = y + \frac{\tau}{\rho^2}\right) d y. \label{eq:gaussianmanipulation1}
    \end{align}
    We then compute
    \begin{align}
        \int_{-\infty}^\infty q_t(y)&\pp\left(\tstar = t \text{ and } \inf_{K(s, t) \leq 1 - \rho^2} \Omega(s) \geq y + \frac{\tau}{\rho^2}| \Omega(t) = y + \frac{\tau}{\rho^2 }\right) dy \label{eq:gaussianaddsubtract}\\
        &= \int_{- \infty}^\infty q_t(y)\pp\left(\tstar = t \text{ and } \inf_{K(s, t) \leq 1 - \rho^2} \Omega(s) \geq y | \Omega(t) = y \right)dy  \\
        &+ \int_{-\infty}^\infty \left(q_t(y) - q_t\left(y - \frac{\tau}{\rho^2 }\right)\right) \cdot \pp\left(\tstar = t \text{ and } \inf_{K(s, t) \leq 1 - \rho^2} \Omega(s) \geq y | \Omega(t) = y \right) dy,
    \end{align}
    where we added and subtracted the first term and then made the variable substitution $y + \frac{\tau}{\rho^2} \mapsto y$ for the latter integral.  Note that
    \begin{align}
        \pp\left(\tstar = t \text{ and } \inf_{K(s, t) \leq 1 - \rho^2} \Omega(s) \geq y | \Omega(t) = y \right) = \pp\left(\tstar = t | \Omega(t) = y\right)
    \end{align}
    as $\Omega(s) \geq \Omega(\tstar)$ for all $s \in T$ by definition.  Combining this observation with \eqref{eq:gaussianmanipulation1} and \eqref{eq:gaussianaddsubtract} yields
    \begin{align}
        \pp\left(\abs{A(\tstar)} = 0\right) &\geq \sum_{t \in T} \int_{-\infty}^\infty q_t(y) \pp\left(\tstar = t | \Omega(t) = y\right) dy \\
        &- \sum_{t \in T} \int_{-\infty}^\infty \left(q_t(y) - q_t\left(y - \frac{\tau}{\rho^2}\right)\right) \cdot \pp\left(\tstar = t | \Omega(t) = y\right) dy .
    \end{align}
    For the first term, we have
    \begin{align}
        \sum_{t \in T} \int_{-\infty}^\infty q_t(y) \pp\left(\tstar = t | \Omega(t) = y\right) d y &= \sum_{t \in T} \pp\left(\tstar = t\right) = 1.
    \end{align}
    For the second term, using the fact that $1 - e^x \leq x$ for all $x$, we have
    \begin{align}
        q_t(y) - q_t\left(y - \frac{\tau}{\rho^2}\right) &= q_t(y) \left(1 - \exp \left( \frac{(y - m(t))^2}{2\eta^2 K(t,t)} - \frac{\left(y - m(t) - \frac{\tau}{\rho^2}\right)^2}{2 \eta^2 K(t,t)}\right) \right) \\
        &\leq q_t(y) \left(\frac{(y - m(t))^2}{2\eta^2 K(t,t)} - \frac{\left(y - m(t) - \frac{\tau}{\rho^2}\right)^2}{2 \eta^2 K(t,t)}  \right) \\
        &\leq \frac{q_t(y)}{2 \eta^2 \kappa^2} \cdot \left( \frac{2 \tau}{\rho^2} (y - m(t)) \right).
    \end{align}
    Thus we have
    \begin{align}
        \pp\left(\abs{A(\tstar)} > 0\right) &= 1 - \pp\left(\abs{A(\tstar)} = 0\right) \\
        &\leq \sum_{t \in T} \int_{-\infty}^\infty \frac{q_t(y)}{2 \eta^2\kappa^2} \cdot \left( \frac{2 \tau}{\rho^2} (y - m(t))\right) \cdot \pp\left( \tstar = t | \Omega(t) = y \right) d y \\
        % &\leq \frac{\tau^2}{2 \eta^2\kappa^2\rho^4} + \sum_{t \in T} \int_{-\infty}^\infty \frac{q_t(y)}{2 \eta^2} \cdot \frac{2 \tau}{\rho^2} \cdot (y - m(t))\cdot \pp\left( \tstar = t | \Omega(t) = y \right) d y \\
        &= \frac{\tau}{\rho^2 \eta^2 \kappa^2} \sum_{t \in T} \int_{-\infty}^\infty (y - m(t)) q_t(y) \pp\left( \tstar = t | \Omega(t) = y \right) d y \\
        &= \frac{\tau}{\rho^2 \eta^2 \kappa^2} \cdot \ee\left[ \Omega(\tstar) - m(\tstar)\right] \\
        &\leq  \frac{\tau}{\rho^2 \eta^2 \kappa^2} \cdot \ee\left[\sup_{t \in T} \Omega(t) - m(t)\right].
    \end{align} 
    The result follows by noting that $\eta \cdot \omega(t) = \Omega(t) - m(t)$ for all $t \in T$.
\end{proof}
We now prove a corollary of \Cref{thm:gaussiananticoncentration} that will be useful in the proof of \Cref{prop:banach_gaussian_anti_concentration} and which makes the relationship between $\cE(\rho, \tau)$ and the intuition of distance between $\tstar$ and $s$ more explicit.
\begin{corollary}\label{cor:gaussiananticoncentration_distance}
    Suppose that we are in the situation of \Cref{thm:gaussiananticoncentration} with the additional conditions that $T$ is a subset of a real vector space and that $d(s,t) = \sqrt{K(s-t,s-t)}$.  Let $m': T \to \rr$ denote a mean function such that $\sup_{t \in T} \abs{m(t) - m'(t)} \leq \tau$ and let $\Omega'$ denote the corresponding shifted Gaussian process.  If $t^{\star '} = \argmin_{t \in T} \Omega'(t)$, then
    \begin{align}
        \pp\left(d(\tstar, t^{\star '}) > \rho \right) \leq \frac{8 \tau}{\rho^4 \eta \kappa^2} \cdot \ee\left[\sup_{t \in T} \omega(t)\right].
    \end{align}
\end{corollary}
\begin{proof}
    Note that
    \begin{align}
        d(s, t)^2 = K(s - t, s - t) = K(s,s) + K(t,t) - 2 K(s,t).
    \end{align}
    Let $M = \max(K(\tstar, \tstar), K(t^{\star '},t^{\star '})) \leq 1$ and note that the above implies:
    \begin{align}
        d(\tstar,t^{\star '})^2 \leq 2 M \left(1 - \frac{K(t^{\star '}, t^{\star})}{M}\right).
    \end{align}
    Thus,
    \begin{align}
        \pp\left( d(\tstar, t^{\star '}) > \rho  \right) &\leq \pp\left( 2 M \left( 1 - \frac{K(\tstar, t^{\star '})}{M} \right) > \rho \right) \\
        &\leq \pp\left( 1 - \frac{K(\tstar, t^{\star '})}{M}  > \frac {\rho^2}2 \right) \\
        &\leq 2 \left(  \frac{\tau}{\left( \frac {\rho^2} 2 \right)^2 \eta \kappa^2} \cdot \ee\left[\sup_{t \in T} \omega(t) \right] \right) \\
        &= \frac{8 \tau}{\rho^4 \eta \kappa^2} \cdot \ee\left[\sup_{t \in T} \omega(t)\right],
    \end{align}
    where the last inequality follows by applying a union bound and \Cref{thm:gaussiananticoncentration} to the Gaussian processes $\Omega$ and $\Omega'$ after observing that
    \begin{align}
        \Omega'(\tstar) \leq \Omega'(t^{\star '}) + \tau
    \end{align}
    and similarly for $\Omega(t^{\star '})$.
\end{proof}
We are now ready to prove \Cref{prop:banach_gaussian_anti_concentration} using \Cref{cor:gaussiananticoncentration_distance}.
\begin{proof}[Proof of \Cref{prop:banach_gaussian_anti_concentration}]
    Let $T = \cF$, and observe that
    \begin{align}
        d(f, f')^2 =  K(f- f', f - f') = \ee\left[ \omega(f) - \omega(f')^2 \right] = \norm{f - f'}^2.
    \end{align}
    The result then follows immediately by applying \Cref{cor:gaussiananticoncentration_distance} to the Gaussian process $\Omega(f) = m(f) + \eta \cdot \omega(f)$.
\end{proof}

%% file: body/app_regression_proofs.tex
\section{Analysis of Algorithms \ref{alg:general_alg} and \ref{alg:ftrl_alg}} \label{app:proofs}

In this section we provide the full proofs for \Cref{thm:general_ftpl_ub,thm:cvx_ftrl}, as well as more general statements under different measures of the complexity of the function class $\cF$.  We begin the section by stating formal bounds on the learning and differential privacy guarantees for each algorithm.  In \Cref{app:stability_analysis}, we prove \Cref{lem:anticoncentration_stability}, which is a key technical lemma in the differential privacy guarantee of \Cref{alg:general_alg}; we then continue in \Cref{app:ftrl_stability_analysis} by proving \Cref{lem:ftrl_stability}, which plays an analogous role except in the analysis of \Cref{alg:ftrl_alg}.  In \Cref{app:output_perturb_analysis}, we prove that algorithmic stability in $\norm{\cdot}_m$ can be boosted to differential privacy through \Cref{alg:output_perturb}.  In \Cref{app:final_dp_guarantees} we combine the previous results to give guarantees on the differential privacy of \Cref{alg:general_alg} and \Cref{alg:ftrl_alg}.  We continue in \Cref{app:learning_theory} by applying more standard learning theoretic techniques to demonstrate that both algorithms are PAC learners before concluding the proofs of the main theorems in \Cref{app:concluding_proofs}.  Finally, for the sake of completeness, we prove a norm comparison lemma in \Cref{app:norm_comparison} that was deferred from \Cref{app:learning_theory} for the purpose of continuity.

\subsection{Stability of Algorithm \ref{alg:general_alg} and Proof of Lemma \ref{lem:anticoncentration_stability}}\label{app:stability_analysis}

We break the proof into two parts.  First, we integrate the tail bound from \Cref{prop:banach_gaussian_anti_concentration} to get control on $\ee\left[ \norm{\fbar_j - \fbar_j'}_m \right]$ for any $j \in [J]$.  We then apply a Jensen's inequality and a Chernoff bound to get high probability control on $\norm{\fbar - \fbar'}$.  We begin with the following lemma:
\begin{lemma}\label{lem:stability_in_expectation}
    Let $\fbar_j$ be as in in Line \ref{line:fbar_j} of \Cref{alg:general_alg} and $\fbar_j'$ is defined analogously with respect to $\cD'$, a neighboring dataset to $\cD$.  Then
    \begin{align}
        \ee\left[ \norm{\fbar_j - \fbar_j'}_m^2 \right] \leq \frac{2}{(n \cdot \eta)^{1/3} \kappa^{2/3}} \cdot \left( \ee\left[ \sup_{f \in \cF} \omega_m(f) \right] \right)^{1/3},
    \end{align}
    where the expectation is with respect to $\xi\ind{j}$.
\end{lemma}
\begin{proof}
    By boundedness of $\ell$, it holds that $\abs{L_{\cD}(f) - L_{\cD'}(f)} \leq \frac 1n$ for all $f \in \cF$.  Thus, if $\fbar\ind{1}$ and $\fbar^{(1)'}$ are as in the statement of the corollary, then $L_{\cD}(\fbar^{(1)'}) + \eta \cdot \omega_m\ind{1}(\fbar^{(1)'}) \leq L_\cD(\fbar\ind{1}) + \eta \cdot \omega_m\ind{1}(\fbar\ind{1}) + \frac 1n$.  Plugging in $\tau = \frac 1n$ and applying \Cref{prop:banach_gaussian_anti_concentration} yields
    \begin{align}
        \pp\left( \norm{\fbar_j - \fbar_j'}_m > \rho \right) \leq \frac{8}{n \rho^4 \kappa^2 \eta} \cdot \ee\left[ \sup_{f \in \cF} \omega_m(f) \right].
    \end{align}
    Now, we can integrate the tail bound to get for any $\zeta > 0$,
    \begin{align}
        \ee\left[ \norm{\fbar - \fbar'}_m^2 \right] &\leq \zeta + \int_\zeta^1 2 \rho \pp\left( \norm{\fbar - \fbar'}_m > \rho \right) d \rho \\
        &\leq \zeta + 2 \cdot \int_\zeta^1 \rho \cdot \frac{8}{n \rho^4 \kappa^2 \eta} \cdot \ee\left[ \sup_{f \in \cF} \omega_m(f) \right] d \rho \\
        &\leq \zeta + \frac{8}{ n\zeta^2 \kappa^2 \eta} \cdot \ee\left[ \sup_{f \in \cF} \omega_m(f) \right] \cdot \log\left( \frac 1\zeta \right).
    \end{align}
    Minimizing over $\zeta$ yields the result.
\end{proof}
We now apply a Jensen's inequality and a Chernoff bound to get high probability control on $\norm{\fbar - \fbar'}$.
\begin{proof}[Proof of \Cref{lem:anticoncentration_stability}]
    By Jensen's inequality and the definition of $\fbar$, it holds that
    \begin{align}
        \pp\left( \norm{\fbar - \fbar'}_m > \rho \right) &= \pp\left( \norm{\frac 1J \sum_{j = 1}^J \fbar_j - \fbar_j'}_m > \rho \right) \\
        &= \pp\left( \norm{\frac 1J \sum_{j = 1}^J \fbar_j - \fbar_j'}_m^2 > \rho^2 \right) \\
        &\leq \pp\left( \frac 1J \sum_{j = 1}^J  \norm{\fbar_j - \fbar_j'}_m^2 > \rho^2 \right).
    \end{align}
    Finally, by Hoeffding's inequality \citep{wainwright2019high,van2014probability} and \Cref{lem:stability_in_expectation}, we have that
    \begin{align}
        \delta &\geq \pp\left( \frac 1J \sum_{j = 1}^J  \norm{\fbar_j - \fbar_j'}_m^2 >  \ee\left[ \norm{\fbar_1 - \fbar_1'}_m \right] + \sqrt{\frac{\log\left( \frac 1\delta \right)}{J}} \right) \\
        &\geq \pp\left( \frac 1J \sum_{j = 1}^J  \norm{\fbar_j - \fbar_j'}_m^2 > \frac{2}{(n \cdot \eta)^{1/3} \kappa^{2/3}} \cdot \left( \ee\left[ \sup_{f \in \cF} \omega_m(f) \right] \right)^{1/3} + \sqrt{\frac{\log\left( \frac 1\delta \right)}{J}}\right).
    \end{align}
    The result follows.
\end{proof}

\subsection{Stability of Algorithm \ref{alg:ftrl_alg} and Proof of Lemma \ref{lem:ftrl_stability}}\label{app:ftrl_stability_analysis}

In this section we prove Lemma \ref{lem:ftrl_stability} based on a technique borrowed from online learning and the analysis of Follow the Regularized Leader (FTRL) \citep{gordon1999regret,cesa2006prediction}.
\begin{proof}[Proof of \Cref{lem:ftrl_stability}]
    Let $\cL$ be as in \eqref{eq:reg_def} and $\cL'$ be defined similarly but with $\cD$ replaced by $\cD'$.  Note that $\norm{\cdot}_m^2$ is strongly convex with respect to $\norm{\cdot}_m$ \citep{rockafellar2015convex}.  Thus, by convexity of $\cF$ , it holds that
    \begin{align}
        \cL(\fbar') \geq \cL(\fbar) + \frac{\eta}{2} \norm{\fbar - \fbar'}_m^2.
    \end{align}
    On the other hand,
    \begin{align}
        \cL(\fbar') &= \cL'(\fbar') + \cL(\fbar') - \cL'(\fbar') \\
        &\leq \cL'(\fbar) + \cL(\fbar') - \cL'(\fbar') \\
        &= \cL(\fbar) + \cL(\fbar') - \cL'(\fbar') + \cL'(\fbar) - \cL(\fbar) \\
        &\leq \cL(\fbar) + \frac{2}n .
    \end{align}
    Combining this with the previous display and rearranging yields:
    \begin{align}
        \eta \cdot \norm{\fbar - \fbar'}_m^2 \leq \frac{4 }n.
    \end{align}
    Rearranging again proves the result.
\end{proof}

\subsection{Boosting Stability to Differential Privacy: Proofs from Section \ref{subsec:output_perturb_analysis}}\label{app:output_perturb_analysis}
In this section, we analyze \Cref{alg:output_perturb} and prove that if $\fbar$ is stable in $\norm{\cdot}_m$ then applying the output perturbation yields a differentially private algorithm for standard choices of perturbation distribution $\cQ$.  We also show that \Cref{alg:output_perturb} returns $\fhat$ close to the $\fbar$ with high probability.  This first claim is implied by the following standard concentration bound.
\begin{lemma}\label{lem:perturb_is_small}
    Suppose that $\cQ = \cN(0,1)$ and let $\fhat = \perturb(\fbar, \cQ, \gamma, \cDtil)$ be as in \Cref{alg:output_perturb}.  Then with probability at least $1 - \beta$,
    \begin{align}\label{eq:gaussian_norm_bound}
        \norm{\fbar - \fhat}_m \leq 2\gamma \cdot \sqrt{\log\left( \frac 1\beta \right)}.
    \end{align}
    If $\cQ = \Lap(1)$, then with probability at least $1 - \beta$,
    \begin{align}\label{eq:exp_norm_bound}
        \norm{\fbar - \fhat}_m \leq 2\gamma \cdot \log\left( \frac 1\beta \right) \cdot \sqrt{m}.
    \end{align}
\end{lemma}
\begin{proof}
    By construction, it holds that
    \begin{align}
        \norm{\fhat - \fbar - \gamma \cdot \zeta}_m \leq \norm{\fbar - \fbar - \gamma \cdot \zeta}_m = \gamma \cdot \norm{\zeta}_m.
    \end{align}
    By the triangle inequality, it holds that
    \begin{align}
        \norm{\fhat - \fbar}_m \leq \norm{\fhat - \fbar - \gamma \cdot \zeta}_m + \gamma \cdot \norm{\zeta}_m \leq 2 \gamma \cdot \norm{\zeta}_m.
    \end{align}
    Thus it suffices to bound $\norm{\zeta}_m$.  Bounds on this quantity when $\zeta \sim \cN(0,1)$ or $\zeta \sim \Lap(1)$ are standard and can be found in, for example \citet{wainwright2019high}.  The result follows.
\end{proof}
We now prove the main property of \Cref{alg:output_perturb}, namely that it boosts stability to differential privacy.
\begin{lemma}\label{lem:perturb_boosting}
    Let $\perturb$ be as in \Cref{alg:output_perturb} and suppose that $\fbar, \fbar' \in \cF$.  If $\cDtil = \left\{ Z_1, \dots, Z_m \right\} \subset \cX$ is arbitrary and $\cQ = \cN(0,1)$, then for any $\delta > 0$ and any measurable $\cG \subset \cF$, it holds that
    \begin{align}
        \pp\left( \perturb(\fbar, \cQ, \gamma, \cDtil) \in \cG \right) \leq e^{\frac{m}{2 \gamma^2}\left(1 + \gamma \cdot \sqrt{\log\left( \frac 1\delta \right)}\right) \cdot \norm{\fbar - \fbar'}_m} \cdot \pp\left( \perturb(\fbar', \cQ, \gamma, \cDtil) \in \cG \right) + \delta.
    \end{align} 
    On the other hand, if $\cQ = \Lap(1)$, then 
    \begin{align}
        \pp\left( \perturb(\fbar, \cQ, \gamma, \cDtil) \in \cG \right) \leq e^{\frac{m^{3/2}}{\gamma} \cdot \norm{\fbar' - \fbar}_m} \cdot \pp\left( \perturb(\fbar', \cQ, \gamma, \cDtil) \in \cG \right) 
    \end{align}
\end{lemma}
\begin{proof}
    To prove the first statement, let $\cB_\delta$ denote the event that $\norm{\zeta}_m \leq \gamma \cdot \sqrt{\log\left( \frac 1\delta \right)}$ and let
    \begin{align}
        p(u) = (2 \pi \gamma^2)^{-\frac m2} \cdot \exp\left( - \frac 1{2 \gamma^2} \cdot \norm{u}_m^2 \right)
    \end{align}
    be the density of $\gamma \cdot \zeta$.  Note that $\pp\left( \cB_\delta^c \right) \leq \delta$ by \Cref{lem:perturb_is_small}.  We compute:
    \begin{align}
        \pp\left( \perturb(\fbar, \cQ, \gamma, \cDtil) \in \cG \right) &= \int_{u \in \rr^m} \pp\left( \perturb(\fbar, \cQ, \gamma, \cDtil) \in \cG | \gamma \cdot \zeta = u \right) p(u) d u \\
        &= \int_{u \in \rr^m} \bbI\left[ \cB_\delta \right] \cdot \pp\left( \perturb(\fbar, \cQ, \gamma, \cDtil) \in \cG | \gamma \cdot \zeta = u \right) p(u) d u \\
        &\quad + \int_{u \in \rr^m} \bbI\left[ \cB_\delta^c \right] \pp\left( \perturb(\fbar, \cQ, \gamma, \cDtil) \in \cG | \gamma \cdot \zeta = u \right) p(u) d u \\
        &\leq \delta + \int_{u \in \rr^m} \bbI\left[ \cB_\delta \right] \cdot \pp\left( \perturb(\fbar, \cQ, \gamma, \cDtil) \in \cG | \gamma \cdot \zeta = u \right) p(u) d u.
    \end{align}
    For the second term, we compute:
    \begin{align}
        \int_{u \in \rr^m} &\bbI\left[ \cB_\delta \right] \cdot \pp\left( \perturb(\fbar, \cQ, \gamma, \cDtil) \in \cG | \gamma \cdot \zeta = u \right) p(u) d u \\
        &= \int_{u \in \rr^m} \bbI\left[ \cB_\delta \right] \cdot \pp\left( \perturb(\fbar', \cQ, \gamma, \cDtil) \in \cG | \gamma \cdot \zeta = u + \fbar' - \fbar \right) p(u) d u \\
        &= \int_{u \in \rr^m} \bbI\left[ \cB_\delta \right] \cdot \pp\left( \perturb(\fbar', \cQ, \gamma, \cDtil) \in \cG | \gamma \cdot \zeta = u \right) p(u) \cdot e^{\frac{m}{2 \gamma^2}\left( \norm{u - \fbar'}_m^2 - \norm{u - \fbar}_m^2 \right)} d u \\
        &\leq \int_{u \in \rr^m} \bbI\left[ \cB_\delta \right] \cdot \pp\left( \perturb(\fbar', \cQ, \gamma, \cDtil) \in \cG | \gamma \cdot \zeta = u \right) p(u) \cdot e^{\frac{m}{2 \gamma^2}\left(1 + \norm{u}_m \right) \cdot \norm{\fbar - \fbar'}_m} d u \\
        &\leq \int_{u \in \rr^m} \bbI\left[ \cB_\delta \right] \cdot \pp\left( \perturb(\fbar', \cQ, \gamma, \cDtil) \in \cG | \gamma \cdot \zeta = u \right) p(u) \cdot e^{\frac{m}{2 \gamma^2}\left(1 + \gamma \cdot \sqrt{\log\left( \frac 1\delta \right)} \right) \cdot \norm{\fbar - \fbar'}_m} d u \\
        &\leq e^{\frac{m}{2 \gamma^2}\left(1 + \gamma \cdot \sqrt{\log\left( \frac 1\delta \right)} \right) \cdot \norm{\fbar - \fbar'}_m} \cdot \pp\left( \perturb(\fbar', \cQ, \gamma, \cDtil) \in \cG \right).
    \end{align}
    The first claim follows.  To prove the second claim, we may repeat the same argument with $\cQ = \Lap(1)$ and $\delta = 0$.  Indeed, observe that if
    \begin{align}
        q(u) = \gamma^{-m} \cdot e^{- \frac{\norm{u}_{\ell^1}}{\gamma}}
    \end{align}
    then
    \begin{align}
        \frac{q(u + \fbar' - \fbar)}{q(u)} \leq e^{\frac{\norm{\fbar' - \fbar}_{\ell^1}}{\gamma}} \leq e^{\frac{m^{3/2}}{\gamma} \cdot \norm{\fbar' - \fbar}_m},
    \end{align}
    where the second inequality follows from Cauchy-Schwarz.  Plugging this ratio into the above argument yields the second claim.
\end{proof}

\subsection{Concluding the Proofs of Differential Privacy}\label{app:final_dp_guarantees}

In this section, we combine the results from \Cref{app:output_perturb_analysis} with the stability results from \Cref{app:concluding_proofs} to prove the differential privacy guarantees of \Cref{alg:general_alg} and \Cref{alg:ftrl_alg}.  We separate this section into two lemmas, each corresponding to one of the algorithms.  We begin with the more general result.  Note that this lemma does not quite follow immediately from combining the stability guarantee with the results of \Cref{app:output_perturb_analysis} as we wish to assume a uniform lower bound on $\norm{f}_\mu$ whereas \Cref{lem:anticoncentration_stability} requires a uniform lower bound on $\norm{\cdot}_m$.  We apply a result from \Cref{app:learning_theory} below to reconcile this discrepancy.
\begin{lemma}[Differential Privacy Guarantee for \Cref{alg:general_alg}] \label{lem:dp_guarantee_ftpl}
    Suppose that $\cF: \cX \to [-1,1]$ is a function class and let $\mu \in \Delta(\cX)$ such that $\norm{f}^2 \geq \frac 23$ for all $f \in \cF$.  Suppose further that $\ell$ is bounded in $[0,1]$.  Let $\delta > 0$ and suppose that $m, \gamma, \eta$, and  $J$ are such that
    \begin{align}
        \frac{m}{2 \gamma^2}\left( 1 + \gamma \cdot \sqrt{\log\left( \frac 2\delta \right)} \right) \left( \frac{4}{(n \cdot \eta)^{1/3}} \cdot \ee\left[ \sup_{f \in \cF} \omega_m(f) \right] + \sqrt{\frac{\log\left( \frac 2\delta \right)}{J}} \right) \leq \epsilon
    \end{align}
    and
    \begin{align}
        C \left( \frac{\log^2(m)}{\sqrt{m}} \cdot \cGbar_m(\cF) + \sqrt{\frac{\log\log(m) + \log\left( \frac 1\delta \right)}{m} }  \right) \leq \frac 16.
    \end{align}
    Then \Cref{alg:general_alg} is $(\epsilon, \delta)$-differentially private for $\cQ = \cN(0,1)$.
\end{lemma}
\begin{proof}
    The result follows immediately by combining \Cref{lem:anticoncentration_stability,lem:perturb_boosting} assuming we have a lower bound on $\kappa$.  Indeed, by \Cref{lem:norm_comparison}, it holds with probability at least $1 - \delta$ that $\inf_{f \in \cF} \norm{f}_m \geq \frac 14$.  The result follows.
\end{proof}
We also have a guarantee for the more specialized algorithm.
\begin{lemma}[Differential Privacy Guarantee for \Cref{alg:ftrl_alg}] \label{lem:dp_guarantee_ftrl}
    Suppose that $\cF:\cX \to [-1,1]$ is a convex function class and suppose that $\ell$ is convex and $\lambda$-Lipschitz in its first argument.  If $\delta > 0$ and $m, \gamma, \eta$ are such that
    \begin{align}
        \frac{m}{2 \gamma^2}\left( 1 + \gamma \cdot \sqrt{\log\left( \frac 2\delta \right)} \right) \cdot \frac{2}{\sqrt{\eta\cdot  n}} \leq \epsilon,
    \end{align}
    then \Cref{alg:ftrl_alg} run with $\cQ = \cN(0,1)$ is $(\epsilon, \delta)$-differentially private.  On the other hand, if $\delta= 0$ and
    \begin{align}
        \frac{m^{3/2}}{\gamma} \cdot \frac{2}{\sqrt{\eta \cdot n}} \leq \epsilon,
    \end{align}
    then \Cref{alg:ftrl_alg} run with $\cQ = \Lap(1)$ is $\epsilon$-purely differentially private.
\end{lemma}
\begin{proof}
    This follows immediately by combining \Cref{lem:ftrl_stability,lem:perturb_boosting}.
\end{proof}
These results show that for any choice of $\nu$, \Cref{alg:general_alg,alg:ftrl_alg} are differentially private.  In the next section we show that if $\nu$ is $\sigma$-smooth with respect to $\mu$, then the algorithms are also PAC learners with respect to $\nu$.

\subsection{PAC guarantees for Algorithms \ref{alg:general_alg} and \ref{alg:ftrl_alg}}\label{app:learning_theory}
The previous sections have shown that \Cref{alg:general_alg,alg:ftrl_alg} are differentially private, which comprises the main difficulty of our analysis.  Here we apply standard learning theoretic techniques to show that if $\nu$ is $\sigma$-smooth with respect to $\mu$, then the algorithms are also PAC learners with respect to $\nu$.  This proof rests on three main results: first, we recall a norm comparison guarantee in high probability that allows us to relate $\norm{\cdot}_\mu$ to $\norm{\cdot}_m$; second, we recall a classical uniform deviations bound for empirical processes; and third, we show that the perturbed empirical minimizer $\fbar$ has similar loss to the empirical minimizer $\ferm$ of a loss function as long as the perturbation is not too large.  Combining all three results will result in a PAC learning guarantee for \Cref{alg:general_alg,alg:ftrl_alg}.

We begin with the following lemma, which is a fairly standard result in learning theory.  To state the lemma, we recall from \Cref{def:gaussianccomplexity} that the worst-case Gaussian complexity is defined as
\begin{align}
    \cGbar_m(\cF) = \sup_{Z_1, \dots, Z_m} \ee\left[ \sup_{f \in \cF} \omega_m(f) \right],
\end{align}
We then have the following control on $\norm{\cdot}_\mu$ in terms of $\norm{\cdot}_m$:
\begin{lemma}\label{lem:norm_comparison}
    Suppose that $\cF: \cX \to [-1,1]$ is a function class and let $\mu \in \Delta(\cX)$ with $Z_1, \dots, Z_m \sim \mu$ independent.  Then for any $\beta > 0$, it holds with probability at least $1 - \beta$ that for all $f, f' \in \cF$,
    \begin{align}
        \norm{f - f'}_\mu \leq 2 \cdot \norm{f - f'}_m + C \left( \frac{\log^2(m)}{\sqrt{m}} \cdot \cGbar_m(\cF) + \sqrt{\frac{\log\log(m) + \log\left( \frac 1\beta \right)}{m} }  \right).
    \end{align}
\end{lemma}
Because the proof is relatively standard \citep{bousquet2002concentration,rakhlin2017empirical}, but also a technical digression, we defer it to \Cref{app:norm_comparison} and continue with our arguments.

Our second lemma is a standard uniform deviation bound for empirical processes:
\begin{lemma}\label{lem:uniform_deviations}
    Let $\cF: \cX \to [-1,1]$ be a bounded function class and let $\cD$ denote a data set of $(X_i, Y_i)\sim \nu$ be independent.  Then for any $\beta > 0$, with probability at least $1 - \beta$, it holds that
    \begin{align}
        \sup_{f \in \cF} \abs{L_\cD(f) - L(f)} \leq \frac{6}{\sqrt{n}} \cdot \ee\left[ \sup_{f \in \cF} \omega_n(f) \right] + \sqrt{\frac{2 \log\left( \frac 1\beta \right)}{n}}.
    \end{align}
\end{lemma}
\begin{proof}
    This follows immediately from combining \citet[Theorem 4.10]{wainwright2019high} with \citet[Lemma 7.4]{van2014probability}.
\end{proof}
Finally, we show that the perturbed empirical minimizer $\fbar$ has similar loss to the empirical minimizer $\ferm$ of a loss function as long as the perturbation is not too large.
\begin{lemma}\label{lem:fbar_good}
    Let $\cF: \cX \to [-1,1]$ denote a function class and let $\ell: [-1,1]^{\times 2} \to [0,1]$ denote a bounded loss function convex in the first argument and let $\cD$ denote a dataset of size $n$.  For $\eta > 0$, let $\fbar$ be as in Line \ref{line:fbar_ftpl} of \Cref{alg:general_alg}.  Then for any $\beta > 0$, with probability at least $1 - \beta$,
    \begin{align}\label{eq:fbar_good_ftpl}
        L(\fbar) - \inf_{f \in \cF} L(f) \leq  \frac{12}{\sqrt{n}} \cdot \ee\left[ \sup_{f \in \cF} \omega_n(f) \right] + 2 \cdot \sqrt{\frac{ \log\left( \frac 1\beta \right)}{n}} + 2 \eta \cdot \left( \ee\left[ \sup_{f \in \cF} \omega_m\ind{j}(f) \right] +  \sqrt{\frac{\log\left( \frac 1\beta \right)}{J}}\right).
    \end{align}
    If instead we let $\fbar$ be as in Line \ref{line:fbar_ftrl} of \Cref{alg:ftrl_alg}, then almost surely,
    \begin{align}\label{eq:fbar_good_ftrl}
         L(\fbar) - \inf_{f \in \cF} L(f) \leq  12 \cdot \ee\left[ \sup_{f \in \cF} \omega_n(f) \right] + 2\cdot \sqrt{\frac{ \log\left( \frac 1\beta \right)}{n}} + \eta.
    \end{align}
\end{lemma}
\begin{proof}
    Applying \Cref{lem:small_regularizer} and noting that $\fbar, \ferm \in \cF$ and applying \Cref{lem:uniform_deviations} yields
    \begin{align}
        L(\fbar) - \inf_{f \in \cF} L(f) \leq \frac{12}{\sqrt{n}} \cdot \ee\left[ \sup_{f \in \cF} \omega_n(f) \right] + 2 \cdot \sqrt{\frac{2 \log\left( \frac 1\beta \right)}{n}} + \sup_{f,f' \in \cF} R(f) - R(f').
    \end{align}
    The second statement follows immediately by noting that $0 \leq \norm{f}_m \leq 1$ for all $f \in \cF$ and letting $R(f) = \eta \cdot \norm{f}_m$.  
    
    For the first statement, we note that by convexity of $\ell$, it holds that
    \begin{align}
        L_\cD(\fbar) \leq \frac 1J \cdot \sum_{j = 1}^J L_\cD(\fbar_j) \leq L_\cD(\ferm) + \frac \eta J \cdot \sum_{j = 1}^J \sup_{f \in \cF} \omega\ind{j}(f) - \inf_{f' \in \cF} \omega\ind{j}(f'). 
    \end{align}
    To prove the second statement, we observe that by the Borell-Tsirelson-Ibragimov-Sudakov inequality (see, e.g., \citet[Example 2.30]{wainwright2019high}) and the fact that $\ee\left[ \omega_m\ind{j}(f)^2 \right] \leq 1$ for all $f \in \cF$, that for all $j \in [J]$, with probability at least $1 - \beta$, it holds that
    \begin{align}
        \sup_{f \in \cF} \omega_m\ind{j}(f) \leq \ee\left[ \sup_{f \in \cF} \omega_m\ind{j}(f) \right] + \sqrt{2 \log\left( \frac 1\beta \right)}.
    \end{align}
    Applying symmetry and a Chernoff bound tells us that with probability at least $1 - \beta$ it holds that
    \begin{align}
        \frac 1J \cdot \sum_{j = 1}^J \sup_{f \in \cF} \omega_m\ind{j}(f) - \inf_{f \in \cF} \omega_m\ind{j} \leq 2 \cdot\ee\left[ \sup_{f \in \cF} \omega_m\ind{j}(f) \right] + 2 \cdot \sqrt{\frac{\log\left( \frac 1\beta \right)}{J}}.
    \end{align}
    The result follows.
\end{proof}
Before continuing, we prove \Cref{lem:small_regularizer} from \Cref{sec:analysis}:
\begin{proof}[Proof of \Cref{lem:small_regularizer}]
    For an arbitrary regularizer $R: \cF \to \rr$, if we let $\fbar \in \argmin_{f \in \cF} L_\cD(f) + R(f)$, then by definition $L_\cD(\fbar) + R(\fbar) \leq L_\cD(\ferm) + R(\ferm)$ and so
    \begin{align}\label{eq:osc_R_bound}
        L_{\cD}(\fbar) \leq L_\cD(\ferm) + \sup_{f, f' \in \cF} R(f) - R(f'),
    \end{align}
    where $\ferm \in \argmin_{f \in \cF} L_\cD(f)$ is the ERM.
\end{proof}

Combining these three lemmas yields the following PAC learning guarantee for \Cref{alg:general_alg}.
\begin{lemma}[PAC Learning Guarantee for \Cref{alg:general_alg}]\label{lem:general_alg_pac}
    Suppose that $\cF: \cX \to [-1,1]$ is a function class and $\ell: [-1,1]^{\times 2} \to [0,1]$ is a bounded loss function $\lambda$-Lipschitz and convex in the first argument.  Let $\mu \in \Delta(\cX)$ and suppose that $\nu$ is $\sigma$-smooth with respect to $\mu$.  For any $\beta > 0$ it holds with probability at least $1 - \beta$ that
    \begin{align}
        L(\fhat) - \inf_{f \in \cF} L(f) &\leq \frac{12}{\sqrt{n}} \cdot \ee\left[ \sup_{f \in \cF} \omega_n(f) \right] + 2 \cdot \sqrt{\frac{ \log\left( \frac 1\beta \right)}{n}} \\
        &\quad + 2 \eta \cdot \left( \ee\left[ \sup_{f \in \cF} \omega_m\ind{j}(f) \right] +  \sqrt{\frac{\log\left( \frac 1\beta \right)}{J}}\right) + \frac{4 \lambda \gamma}{\sigma} \cdot \sqrt{\log\left( \frac 1\beta \right)} \\
        &\quad+ \frac{C\lambda}{\sigma}\cdot \left( \frac{\log^3(m)}{\sqrt{m}} \cdot \cGbar_m(\cF) + \sqrt{\frac{\log\log(m) + \log\left( \frac 1\beta \right)}{m} } \right).
    \end{align}
    for $\fhat$ returned by \Cref{alg:general_alg} with $\cQ = \cN(0,1)$ and $\cD$ a dataset of size $n$.
\end{lemma}
\begin{proof}
    We compute:
    \begin{align}
        L(\fhat) &= L(\fbar) + L(\fhat) - L(\fbar) \leq L(\fbar) + \lambda \cdot \norm{\fbar - \fhat}_{\nu_X} \leq L(\fbar) + \frac{\lambda}{\sigma} \cdot \norm{\fbar - \fhat}_\mu,
    \end{align}
    where the first inequality uses Jensen's and the second usesthe fact that $\nu$ is $\sigma$-smooth.  We now observe that by \Cref{lem:perturb_is_small}, with probability at least $1 - \beta$,
    \begin{align}
        \norm{\fbar - \fhat}_\mu \leq 2 \gamma \cdot \sqrt{\log\left( \frac 1\beta \right)}.
    \end{align}
    Combining this with \Cref{lem:fbar_good} yields
    \begin{align}
        L(\fhat) - \inf_{f \in \cF} L(f) &\leq  \frac{12}{\sqrt{n}} \cdot \ee\left[ \sup_{f \in \cF} \omega_n(f) \right] + 2 \cdot \sqrt{\frac{ \log\left( \frac 1\beta \right)}{n}} + 2 \eta \cdot \left( \ee\left[ \sup_{f \in \cF} \omega_m\ind{j}(f) \right] +  \sqrt{\frac{\log\left( \frac 1\beta \right)}{J}}\right) \\
        &\quad + \frac{4 \lambda \gamma}{\sigma} \cdot \sqrt{\log\left( \frac 1\beta \right)} + \frac{\lambda}{\sigma} \left( \norm{\fhat - \fbar}_\mu - \norm{\fhat - \fbar}_m \right).
    \end{align}
    Applying \Cref{lem:norm_comparison} concludes the result.
\end{proof}
Similarly, we have a result for \Cref{alg:ftrl_alg}; note that while convexity of $\ell$ is \emph{required} to demonstrate that \Cref{alg:general_alg} is a PAC learner, although is irrelevant to the privacy guarantee in \Cref{lem:dp_guarantee_ftpl}, the situation for \Cref{alg:ftrl_alg} is reversed in that convexity \emph{is not required} to demonstrate that \Cref{alg:ftrl_alg} is a PAC learner while it is necessary for the privacy guarantee in \Cref{lem:dp_guarantee_ftrl}.
\begin{lemma}[PAC Guarantees for \Cref{alg:ftrl_alg}]\label{lem:ftrl_pac}
    Suppose that $\cF: \cX \to [-1,1]$ is a convex function class and $\ell: [-1,1]^{\times 2} \to [0,1]$ is a bounded loss function $\lambda$-Lipschitz in the first argument.  Let $\mu \in \Delta(\cX)$ and suppose that $\nu$ is $\sigma$-smooth with respect to $\mu$.  For any $\beta > 0$ it holds with probability at least $1 - \beta$ that
    \begin{align}
        L(\fhat) - \inf_{f \in \cF} L(f) &\leq \frac{12}{\sqrt{n}}\cdot \ee\left[ \sup_{f \in \cF} \omega_n(f) \right] + 2 \cdot \sqrt{\frac{ \log\left( \frac 1\beta \right)}{n}} + \eta + \frac{4 \lambda \gamma}{\sigma} \cdot \sqrt{\log\left( \frac 1\beta \right)} \\
        &\quad+ \frac{C\lambda}{\sigma}\cdot \left( \frac{\log^3(m)}{\sqrt{m}} \cdot \cGbar_m(\cF) + \sqrt{\frac{\log\log(m) + \log\left( \frac 1\beta \right)}{m} }  \right).
    \end{align}
    for $\fhat$ returned by \Cref{alg:ftrl_alg} with $\cQ = \cN(0,1)$ and $\cD$ a dataset of size $n$.  Similarly, if we replace $\cQ = \Lap(1)$, then with probability at least $1 - \beta$,
    \begin{align}
        L(\fhat) - \inf_{f \in \cF} L(f) &\leq \frac{12}{\sqrt{n}} \cdot \ee\left[ \sup_{f \in \cF} \omega_n(f) \right] + 2 \cdot \sqrt{\frac{ \log\left( \frac 1\beta \right)}{n}} + \eta + \frac{4 \lambda \gamma}{\sigma} \cdot \log\left( \frac 1\beta \right) \cdot \sqrt{m} \\
        &\quad+ \frac{C\lambda}{\sigma}\cdot \left( \frac{\log^3(m)}{\sqrt{m}} \cdot \cGbar_m(\cF) + \sqrt{\frac{\log\log(m) + \log\left( \frac 1\beta \right)}{m} } \right).
    \end{align}
\end{lemma}
\begin{proof}
    The proof of the first statement is identical to that of \Cref{lem:general_alg_pac} with the exception of replacing \eqref{eq:fbar_good_ftpl} by \eqref{eq:fbar_good_ftrl} in the invocation of \Cref{lem:fbar_good}.  The second statement is also identical but now replacing \eqref{eq:gaussian_norm_bound} by \eqref{eq:exp_norm_bound} when applying \Cref{lem:perturb_is_small}.
\end{proof}
With these results in hand, along with those from \Cref{app:final_dp_guarantees}, all that remains to conclude the proofs of the main theorems is to tune the hyperparameters and control the complexity terms.  We do this in the next section.

\subsection{Concluding the Proofs of Theorems \ref{thm:general_ftpl_ub} and \ref{thm:cvx_ftrl}}\label{app:concluding_proofs}
In this section, we combine the results from \Cref{app:final_dp_guarantees,app:learning_theory} to prove the main theorems.  The main theorems in the text, \Cref{thm:general_ftpl_ub,thm:cvx_ftrl}, follow immediately from the following two results.  We begin by stating a more detailed version of \Cref{thm:general_ftpl_ub}:
\begin{theorem}\label{thm:general_full}
    Suppose that $\cF: \cX \to [-1,1]$ is a function class such that $\cGbar_m(\cF) = O\left( \sqrt{d} \right)$ for some $d \in \bbN$ and $\ell: [-1,1]^{\times 2} \to [0,1]$ is convex and $\lambda$-Lipschitz in the first argument.  If we let $\cQ = \cN(0,1)$ in \Cref{alg:output_perturb} and set
    \begin{align}
        \gamma  &= \Theta\left( \frac{ \sigma \alpha}{\lambda \cdot \sqrt{\log\left( \frac 1\beta \right)}} \right), \qquad \eta = \Theta(\frac{\alpha}{\sqrt d}), \qquad m = \Thetatil\left( \frac{d \vee \log\left( \frac 1\beta \right)}{\sigma^2 \alpha^2} \cdot \lambda^2\right)  \\
        J &= \Omegatil\left( \frac{d \log\left( \frac 1\beta \right)}{\alpha^2} \vee \frac{\lambda^8 d^2 \log^4\left( \frac 1\beta \right) \log\left( \frac 1\delta \right)}{\sigma^8 \alpha^8 \epsilon^2} \vee \frac{d^2 \lambda^6 \log^3\left( \frac 1\beta \right)\log^2\left( \frac 1\delta \right)}{\sigma^6 \alpha^6 \epsilon^2}\right)
    \end{align}
    and
    \begin{align}
        n = \Omegatil\left(\frac{\lambda^{12}d^5 \log^6\left( \frac 1\beta \right)}{\sigma^{12} \epsilon^3 \alpha^{14}} \vee \frac{d^5 \lambda^9 \log^{9/2}\left( \frac 1\beta \right) \log^{3/2}\left( \frac 1\delta \right)}{\sigma^9 \epsilon^3 \alpha^{10}}  \right),
    \end{align}
    then \Cref{alg:general_alg} is $(\epsilon, \delta)$-differentially private and an $(\alpha, \beta)$-PAC learner with respect to any $\nu$ that is $\sigma$-smooth with respect to $\mu$.
\end{theorem}
\begin{proof}
    This follows by combining \Cref{lem:dp_guarantee_ftpl} with \Cref{lem:general_alg_pac} and plugging in the parameter choices.
\end{proof}
We also have a result for \Cref{alg:ftrl_alg}:
\begin{theorem}\label{thm:ftrl_full}
    Suppose that $\cF: \cX \to [-1,1]$ is a convex function class such that $\cGbar_m(\cF) = O\left( \sqrt{d} \right)$ for some $d \in \bbN$ and $\ell: [-1,1]^{\times 2} \to [0,1]$ is convex and $\lambda$-Lipschitz in the first argument.  If we let $\cQ = \cN(0,1)$ in \Cref{alg:output_perturb} and set
    \begin{align}
        \gamma  = \Theta\left( \frac{ \sigma \alpha}{\lambda \cdot \sqrt{\log\left( \frac 1\beta \right)}} \right), && \eta = \Theta(\alpha), && m = \Thetatil\left( \frac{d \vee \log\left( \frac 1\beta \right)}{\sigma^2 \alpha^2} \cdot \lambda^2 \right)
    \end{align}
    and
    \begin{align}
        n = \Omegatil\left( \frac{\lambda^5 d \log^2\left( \frac 1\beta \right)}{\epsilon \sigma^4 \alpha^5} \vee  \frac{\lambda^4 d \log^{3/2}\left( \frac 1\beta \right) \log^{1/2}\left( \frac 1\delta \right) }{\epsilon \sigma^3 \alpha^4} \right),
    \end{align}
    then \Cref{alg:ftrl_alg} is $(\epsilon, \delta)$-differentially private and an $(\alpha, \beta)$-PAC learner with respect to any $\nu$ that is $\sigma$-smooth with respect to $\mu$.

    On the other hand, if we set $\cQ = \Lap(1)$ in \Cref{alg:output_perturb} and set
    \begin{align}
        \gamma  = \Thetatil\left( \frac{ \alpha^2 \sigma}{\lambda^2 \cdot \left(d \sqrt{\log\left( \frac 1\beta \right)} \wedge \log\left( \frac 1\beta \right)\right)} \right), && \eta = \Theta(\alpha), && m = \Thetatil\left( \frac{d \vee \log\left( \frac 1\beta \right)}{\sigma^2 \alpha^2} \cdot \lambda^2 \right)
    \end{align}
    and
    \begin{align}
        n = \cLtil\left( \frac{\lambda^6}{\sigma^5 \epsilon \alpha^6} \cdot \left( d^2 \sqrt{\log\left( \frac 1\beta \right)} \vee \log^{5/2}\left( \frac 1\beta \right) \right) \right)
    \end{align}
    then \Cref{alg:ftrl_alg} is $\epsilon$-purely differentially private and an $(\alpha, \beta)$-PAC learner with respect to any $\sigma$-smooth $\nu$.
\end{theorem}
\begin{proof}
    This follows immediately by combining \Cref{lem:dp_guarantee_ftrl} with \Cref{lem:ftrl_pac} and plugging in the parameter choices, with the pure differential privacy guarantees coming from the second halves of each lemma.
\end{proof}
As a final remark, we note that \Cref{lem:general_alg_pac,lem:ftrl_pac} are both phrased entirely in terms of $\cGbar_m(\cF)$ and thus apply to function classes $\cF$ such that $\cGbar_m(\cF) = \omega(1)$.  Such \emph{non-donsker} \citep{wainwright2019high,van2014probability} classes can still be learned in the PAC framework, albeit with slower rates.  Indeed, it is immediate from the above results that as long as $\cGbar_m(\cF) = o(\sqrt{m})$, then appropriately tuning the hyperparamters results \Cref{alg:general_alg,alg:ftrl_alg} being differentially private PAC learners with respect to $\nu$.  It is well-known that (with our scaling) $\cGbar_m(\cF) = o(\sqrt{m})$ is a necessary condition for PAC learnability even absent a privacy condition \citep{wainwright2019high,van2014probability} and thus our results qualitatively demonstrate that private learnability with public data is possible whenever non-private learning is possible.

\subsection{Proof of Lemma \ref{lem:norm_comparison}}\label{app:norm_comparison}

Replacing $\cF$ by $\cF - \cF = \left\{ f - f' | f,f' \in \cF \right\}$ and noting that the uniform bound only increases by a factor of $2$ and the Rademacher complexity increases at most by a factor of 2, we observe that it suffices to prove the result for $f' = 0$.  We thus instead prove the notationally simpler claim that with probability at least $1 - \beta$, for all $f \in \cF$,
\begin{align}
    \norm{f}_\mu \leq 2 \cdot \norm{f}_m + C \left( \frac{\log^3(m)}{\sqrt{m}} \cdot \cGbar_m(\cF) + \frac{\log\log(m) + \log\left( \frac 1\beta \right)}{m}   \right)
\end{align}
We first note that by \citet[Theorem 6.1]{bousquet2002concentration}, with probability at least $1 - \beta$, it holds that for all $f \in \cF$,
\begin{align}\label{eq:bousquet}
    \norm{f}_\mu^2 \leq 2 \cdot \norm{f}_m^2 + 200 \left( \rbar^2 + \frac{\log\left( \frac 1\beta \right) + \log\log(m)}{m} \right)
\end{align}
for some universal constant $C$, with
\begin{align}\label{eq:rbar_def}
    \rbar \leq \inf\left\{ r > 0 | \ee_\xi\left[ \sup_{\substack{f \in \cF \\ \norm{f}_m^2 \leq r^2}} \frac 1m \cdot \sum_{i = 1}^m \xi_i f(Z_i)^2 \right] \leq \frac{r^2}{2} \right\}.
\end{align}
Taking square roots on both sides of \eqref{eq:bousquet} shows that it suffices to upper bound $\rbar$.  For the remainder of the proof, we do this.

In order to proceed, we recall the following standard definition of covering numbers.
\begin{definition}
    Let $\cF$ be a function class and $\norm{\cdot}_{m, \infty}$ the the $L^\infty$ norm on the empirical measure on $Z_1, \dots, Z_m \in \cX$, i.e., $\norm{f}_{m,\infty} = \max_{i \in [m]} \abs{f(Z_i)}$.  We say that $f_1, \dots, f_N$ is an $\epsilon$-cover with respect to $\norm{\cdot}_{m,\infty}$ if for all $f \in \cF$ there is some $f_j$ such that $\norm{f - f_j}_{m,\infty} \leq \epsilon$.  We then let $\cN_{m,\infty}(\cF, \epsilon)$ denote the size of the smallest $\epsilon$-cover of $\cF$ with respect to $\norm{\cdot}_{m,\infty}$.
\end{definition}
The notion of a cover is standard throughout learning theory and can be used to control the Rademacher and Gaussian complexities \citep{dudley1969speed,van2014probability,wainwright2019high}.  We will use it to control $\rbar$.

Proceeding with the proof, let $\cN_{m,\infty}(\cF, u)$ denote the covering number of the function class $\cF$ with respect to $\norm{\cdot}_{m,\infty}$ at scale $u > 0$.  We then claim that $\rbar$ can be upper bounded by any $r$ satisfying
\begin{align}\label{eq:fixed_point_covering_number}
    \frac{50}{\sqrt m} \cdot \int_{r/16}^1 \sqrt{\log \cN_{m,\infty}(\cF, u)} du \leq r.
\end{align}

We also claim that for any $r > \cGbar_m(\cF)$, the following holds:
\begin{align}\label{eq:worst_case_gaussian_complexity}
    \int_{r/16}^1 \sqrt{\log \cN_{m,\infty}(\cF, u)} du \leq C \sqrt{\log(m)} \cdot \left( \int_{r}^1 \frac{\sqrt{\log\left( \frac{cm}u \right)}}{u}d u \right) \cdot \cGbar_m(\cF)
\end{align}
for some universal constant $C$.  We now suppose that \eqref{eq:fixed_point_covering_number} and \eqref{eq:worst_case_gaussian_complexity} hold and set $r = C  m^{-1/2} \cdot \log^3(m) \cdot \cGbar_m(\cF)$.  Then it is immediate that $r$ is a member of the set in \eqref{eq:rbar_def}.  

We now prove the two claims.

\paragraph{Proof that a solution to \eqref{eq:fixed_point_covering_number} is an upper bound on $\rbar$.} By a standard Dudley Chaining argument \citep{van2014probability,wainwright2019high}, it holds that
\begin{align}\label{eq:dudley}
    \ee_\xi\left[ \sup_{\substack{f \in \cF \\ \norm{f}_m^2 \leq r^2}} \frac 1m \cdot \sum_{i = 1}^m \xi_i f(Z_i)^2 \right]  \leq \inf_{u > 0} \left\{ 4 u + \frac{12}{\sqrt m} \cdot \int_u^r \sqrt{\log \cN_{m,\infty}(\cF^2 \cap \left\{ \norm{f}_m^2 \leq r^2 \right\}, u)} d u \right\}.
\end{align}
Letting $f_1, \dots, f_M \in \cF \cap \left\{ \norm{f}_m^2 \leq r^2 \right\}$ be a \emph{proper} $u$-cover of $\cF \cap \left\{ \norm{f}_m^2 \leq r^2 \right\}$ with respect to $\norm{\cdot}_{m,\infty}$ at scale $s \leq r$ and $\pi : \cF \to \left\{ f_i \right\}$ be projection to the cover, we have
\begin{align}
    \norm{f^2 - \pi(f)^2}_m^2 \leq s^2 \cdot \norm{f + \pi(f)}_m^2 \leq 4 s^2 r^2
\end{align}
by factoring $f^2 - \pi(f)^2 = (f - \pi(f)) (f + \pi(f))$.  In particular,
\begin{align}
    \cN_{m,2}(\cF^2 \cap \left\{ \norm{f}_m^2 \leq r^2 \right\}, 2 u r)  \leq \cN_{m,\infty}(\cF \cap \left\{ \cF^2 \cap \left\{ \norm{f}_m^2 \leq r^2 \right\} \right\}, u) \leq \cN_{m,\infty}(\cF, u),
\end{align}
where we used the fact that a proper covering at scale $\epsilon$ has size bounded by a covering at scale $\epsilon / 2$ by the triangle inequaltiy.  Substituting into \eqref{eq:dudley} and rescaling yields the claim.

\paragraph{Proof that \eqref{eq:worst_case_gaussian_complexity} holds.}  This proof goes through fat shattering numbers, a complexity measure taking a function class and a scale $u > 0$ and returns $\fat(\cF, u) \in \bbN$  \citep{bartlett1994fat}.  We do not need the full definition of fat shattering numbers and defer to \citep{bartlett1994fat,srebro2010smoothness,rudelson2006combinatorics} for details.  We only need the following two properties.  First, for any $u > 0$, it holds by \citet{rudelson2006combinatorics} that
\begin{align}\label{eq:rudelson}
    \log \cN_{m,\infty}(\cF, u) \leq C \cdot \fat(\cF, cu) \cdot \log(m) \cdot \log\left( \frac{m}{\fat(\cF, c u) u} \right).
\end{align}
Second, by \citet[Lemma A.2]{srebro2010smoothness}, for all $r > m^{-1/2} \cdot \cGbar_m(\cF)$,
\begin{align}\label{eq:srebro}
    r^2 \cdot \fat(\cF, r) \leq 4 \cdot \cGbar_m(\cF)^2.
\end{align}
Plugging \eqref{eq:rudelson} into the left hand side of \eqref{eq:worst_case_gaussian_complexity} and then applying \eqref{eq:srebro} concludes the proof of the claim.

%% file: body/classification.tex
\section{Classification and Analysis of RRSPM} \label{app:classification}

In this section, we provide full proofs for the guarantees of Algorithm~\ref{alg:RRSPM}. In Section~\ref{app:separatorsets}, we describe the concept of universal identification sets and the result of \cite{neel2019use} which plays a crucial role in the privacy analysis of our algorithm. In Section~\ref{app:classification_privacy}, we formally prove Lemma~\ref{lem:projpriv_body} which is the technical lemma used for the proof of differential privacy based on \cite{neel2019use}. We then continue in Section~\ref{app:classifcation_accuracy} by applying standard learning theoretic techniques to demonstrate that our algorithm is an accurate classifier. 

\subsection{Universal Identification Set based Algorithm}\label{app:separatorsets}
In this section, we formally define universal identification sets and informally describe the algorithm used in \cite{neel2019use}. Intuitively, a universal identification set captures the combinatorial property of a function class that all distinct functions in the function class disagree on at least one point from the data universe. For many natural classes, the size of the universal identification set is proportional to the VC dimension of the function class. 

We now describe the algorithm from \cite{neel2019use} and explain the usefulness of universal identification sets. First, we formally define the notion of universal identification set.

\begin{definition}(Universal Identification Set)
    A set $\cU\subseteq \cX$ is a universal identification set for a hypothesis class $\cF$ if for all pairs of functions $f,f'$ in the hypothesis class $\cF$, there is a $x\in \cU$ such that: 
    $$f(x)\neq f'(x).$$

    Additionally, if $\vert \cU\vert =m$, we say that $\cF$ has a universal identification set of size $m$.
\end{definition}

Assuming the existence of a universal identification set for the function class $\cF$ of size $m$ denoted by $\cU=\{U_1,\dots ,U_m\}$, \cite{neel2019use} showed that the following algorithm(called $\mathsf{RSPM}$) is an $\varepsilon$- pure differentially private and $(\alpha,\beta)$-accurate algorithm.

\begin{algorithm}[!ht]
\caption{$\mathsf{RSPM}$}
    \begin{algorithmic}[1]
    \State \textbf{Input } ERM oracle $\erm$, dataset $\cD=\{\left(X_i,Y_i\right)\mid 1\leq i\leq n\}$, hypothesis class $\cF$, universal identification set $\cU=\{U_1,\dots ,U_m\}$, loss function $\ell:\cY\times\cY\to \{0,1\}$.
    \State \textbf{Draw } weights $\ww=\{\xi_1,\dots,\xi_m\}$ such that $\xi_i\sim \lap(2m/\varepsilon)$.
    \State \textbf{Draw } labels $\tilde{Y}=\{\tilde{Y}_1,\dots,\tilde{Y}_m\}$ such that $\tilde{Y}_i\sim \mathrm{Uni}(\{0,1\})$.
    \State \textbf{Define } $\cL_{\ww,\cD,\cU}: \cF \to \rr$ such that
            \begin{align}
                \cL_{\ww,\cD,\cU}(f) = \sum_{i=1}^{n} \ell(f(X_i), Y_i) + \sum_{i=1}^m\xi_i \cdot \ell(f(U_i),\tilde{Y}_i).
            \end{align}
    \State \textbf{Get} $\hat{f}=\erm(\cF,\cL_{\ww,\cD,\cU})$.
    \end{algorithmic}
\end{algorithm}

$\rspm$ roughly simulates ``Report-Noisy-Min''(\cite{dwork2014algorithmic}) attempting to output a function that minimizes a perturbed estimate, where the perturbation is sampled from a Laplace distribution. A straight forward implementation of ``Report-Noisy-Min" to minimize over all perturbed estimates of functions, it'd have to check for all functions in $\cF$ and thus the computational complexity would depend on the size of $\cF$. $\rspm$ avoids this problem by implicitly perturbing the function evaluations via an augmented dataset. The proof of privacy thus exploits the structure of the universal identification set.

Although many natural function classes have bounded universal identification sets, their existence is not as general as having bounded VC dimension. In our work, we only assume finite VC dimension of the function class.

\subsection{Privacy Analysis}\label{app:classification_privacy}
    In this section, we prove that Algorithm~\ref{alg:RRSPM} is differentially private. A notion that we will need is that of a projection of a hypothesis class onto a set of points from the domain. 
\begin{definition}(Projection)
    Given a hypothesis class $\cF\subseteq \cY^{\vert\cX\vert}$ and a subset $Z=\{z_1,\dots z_m\}$ of the feature space $\cX$, we define the projection of $\cF$ onto $Z$ to be $\proj{\cF}{Z}=\{\left(f(z_1),\dots ,f(z_m)\right):f\in \cF\}$. 
\end{definition}
In the following sections, we will interchangeably think of the projection of a hypothesis class onto a set of points $Z$ as a set of functions on $Z$ or as a set of vectors in $ \cY^{\abs{Z}} $.  

By construction, $\proj{\cF}{Z}$ has the property that any two distinct functions $f \ne f' \in \proj{\cF}{Z}$ must disagree on at least one point $z \in Z$. We encapsulate this as the following lemma.

%An easy to see but important property of the projection of a hypothesis class onto a set of points is that for any two functions $f,f'\in \proj{\cF}{Z}$, there exists a point $z\in Z$ such that $f(z)\neq f'(z)$. 
%We state this as a lemma below.
%
\begin{lemma} \label{lem:proj}
    Let $\cF$ be a hypothesis class and let $Z=\{z_1,\dots ,z_m\}$ be a set of points in the instance space $\cX$. Then, for all $f \ne f'\in \proj{\cF}{Z}$, there exists $z\in Z$ such that $f(z)\neq f'(z)$.
\end{lemma}
\begin{remark}\label{re:projuis}
    This property is analogous to the notion of a universal identification set considered in \cite{neel2019use}. 
    In particular, the above lemma can be seen as the statement that the set $Z$ is a universal identification set for the class $\proj{\cF}{Z}$.
\end{remark}

    We first provide an informal sketch of the proof of privacy. In the later sections, we formalize these ideas. Let $\hat{f}\in\cF$ be any arbitrary function and let $\cD,\cD'$ be any pair of neighbouring datasets. We show that $\pp(\rrspm(\cD)=\hat{f})\leq e^{\varepsilon}\pp(\rrspm(\cD')=\hat{f})$. By the definition of the projection, there is some $\tilde{f}\in\proj{\cF}{\tilde{D}_x}$ that is consistent with the labelling of the selected function $\hat{f}$. $\pp(\erm(\cF,\cL_{\ww,\cD,\tilde{\cD}})=\tilde{f})\leq e^{\varepsilon}\pp(\erm(\cF,\cL_{\ww,\cD',\tilde{\cD}})=\tilde{f})$ since privacy for $\hat{f}$ follows from the post-processing property of differential privacy.
    
    We now provide with a proof sketch for the main technical lemma showing that the privacy is preserved over the projected function class. Optimizing a loss function perturbed by Laplace-weighted examples implicitly tries to implement ``Report-Noisy-Min'' algorithm outputting a function that minimizes a perturbed estimate.  For any neighbouring datasets $\cD$ and $\cD'$, the evaluation of any function $\tilde{f}$ can differ by at most $1$. We show that the set of public points $\tilde{\cD}$ is a universal identification set for the set of functions projected onto $\tilde{\cD}$ and leverage this to prove that whenever the shift in the noise vectors is bounded by $2$ in every coordinate, then $\tilde{f}$ is the minimizing function when switching from $\cD$ to $\cD'$. This intuition is made precise in Lemma~\ref{lem:adjacent}.

    \begin{lemma}\label{lem:adjacent}
     Let $\cD, \cD'$ be two neighbouring data sets, and let $\tilde{\cD}=(\tilde{\cD}_x,\tilde{\cD}_y)\in(\cX\times\{0,1\})^m$ where $\tilde{\cD}_x=\{Z_1,\dots ,Z_m\}$ and $\tilde{\cD}_y=\{\tilde{Y}_1,\dots ,\tilde{Y}_m\}$. Define $\mathcal{E}(\projected{f},\cD,\tilde{\cD}) = \left\{ \ww : \erm(\proj{\cF}{\tilde{\cD}_x},\cL_{\ww,\cD,\tilde{\cD}}) = \projected{f}\right\}$, where $\cL_{\ww,\cD,\tilde{\cD}}$ is a functional as defined below:
         \begin{align}
            \cL_{\ww,\cD,\tilde{\cD}}(f) = \sum_{i=1}^{n} \ell(f(X_i), Y_i) + \sum_{i=1}^m\xi_i \cdot \ell(f(Z_i),\tilde{Y}_i).
        \end{align}
      Let $\ww=\{\xi_1,\dots,\xi_m\}$ such that $\xi_i\sim \lap(2m/\varepsilon)$. Given a fixed $\projected{f}\in\proj{\cF}{\tilde{\cD}_x}$, define a mapping $\map_{\projected{f}}(\ww):\rr^m\to\rr^m$ on noise vectors as follows:
     \begin{enumerate}
         \item if $\ell(\projected{f}(Z_i),\tilde{Y}_i)=1,\map_{\projected{f}}(\ww)_i=\xi_i-2$
         \item if $\ell(\projected{f}(Z_i),\tilde{Y}_i)=0,\map_{\projected{f}}(\ww)_i=\xi_i+2$
     \end{enumerate}
    Equivalently, $\map_{\projected{f}}(\ww)_i = \xi_i + 2(1 - 2\ell(\projected{f}(Z_i),\tilde{Y}_i))$.
    Let $\ww \in  \mathcal{E} \left(\projected{f},\cD,\tilde{\cD}\right) $  where $\projected{f}\in\erm(\proj{\cF}{\tilde{\cD}_x},\cL_{\ww,\cD,\tilde{\cD}})$.
    Then $\map_{\projected{f}}(\ww) \in \mathcal{E}(\projected{f},\cD',\tilde{\cD})$.    
    \end{lemma}

    \begin{proof}
        Let $\map_{\projected{f}}(\ww)=\ww'=(\xi_1',\dots \xi_m')$. Our goal is to show that for every $f \in \proj{\cF}{\tilde{\cD}_x}$ such that $f\neq \projected{f}$, we have $\cL_{\ww',\cD',\tilde{\cD}}(f)>\cL_{\ww',\cD',\tilde{\cD}}(\projected{f})$. First, recall that by our assumption for all $f\in\proj{\cF}{\tilde{\cD}_x}$, we have
        \begin{equation}
            \cL_{\ww,\cD,\tilde{\cD}}(f)>\cL_{\ww,\cD,\tilde{\cD}}(\projected{f}). \label{ref:optf}
        \end{equation}

        We now argue that $\cL_{\ww',\cD',\tilde{\cD}}(f)-\cL_{\ww',\cD',\tilde{\cD}}(\projected{f})$ is strictly positive for all $f\in\proj{\cF}{\tilde{\cD}_x}$ such that $\projected{f}\neq f$. To see this we calculate,
        \begin{align}
            \cL_{\ww',\cD',\tilde{\cD}}(f)-\cL_{\ww',\cD',\tilde{\cD}}(\projected{f})
            &= \sum_{(X,Y)\in\cD'} \ell(f(X), Y) + \sum_{i=1}^m\xi_i' \cdot \ell(f(Z_i), \tilde{Y}_i)\\ &- \sum_{(X,Y)\in\cD'}
            \ell(\projected{f}(X), Y) - \sum_{i=1}^m\xi_i' \cdot \ell(\projected{f}(Z_i),\tilde{Y}_i)\\
            &\geq\cL_{\ww,\cD,\tilde{\cD}}(f) -1 + \sum_{i=1}^m\xi_i' \cdot \ell(f(Z_i), \tilde{Y}_i)-\xi_i \cdot \ell(f(Z_i), \tilde{Y}_i)\\
            &-\cL_{\ww,\cD,\tilde{\cD}}(\projected{f}) - 1 - \sum_{i=1}^m\xi_i' \cdot \ell(\projected{f}(Z_i), \tilde{Y}_i)+\xi_i \cdot \ell(\projected{f}(Z_i), \tilde{Y}_i)\\
            &> -2 + \sum_{i=1}^m(\xi_i' - \xi_i)\left(\ell(f(Z_i), \tilde{Y}_i))-\ell(\projected{f}(Z_i), \tilde{Y}_i)\right),
        \end{align}

        where the second inequality follows from the fact that $\cD$ and $\cD'$ differ in only one entry and $\ell$ is $1-$ sensitive. The last equation follows from statement \ref{ref:optf}.
  
        We know from \cref{lem:proj} that there exists a $Z\in\tilde{\cD}$ such that $f(Z)\neq \projected{f}(Z)$.
        Recall that  $\xi_i' = \xi_i + 2(1 - 2\ell(\projected{f}(Z_i),\tilde{Y}_i))$. By construction, each term is non-negative. 
        Therefore,
        \[
            \sum_{i=1}^m(\xi_i' - \xi_i)\left(\ell(f(Z_i), \tilde{Y}_i))-\ell(\projected{f}(Z_i), \tilde{Y}_i)\right)>2.
        \]
        
        To wrap up, we can bound
        \[
          \cL_{\ww',\cD',\tilde{\cD}}(f)-\cL_{\ww',\cD',\tilde{\cD}}(\projected{f})  > 0.
        \]

        This proves that $\map_{\projected{f}}(\ww) \in \mathcal{E}(\projected{f},\cD',\tilde{\cD})$.
        
    \end{proof}

\begin{lemma}[Laplace shift] \label{lem:laplace}
    Let $\ww=\{\xi_1,\dots,\xi_m\}$ such that $\xi_i\sim \lap(2m/\varepsilon)$. Fix some noise realization $\mathbf{r}\in \rr^m$ and fix a hypothesis in the projection set $f\in\proj{\cF}{\tilde{\cD}_x}$. Then,
    \begin{align}
        \pp(\ww=\mathbf{r})\leq e^{\varepsilon}\pp(\ww=\map_{f}(\mathbf{r})),
    \end{align}

    where $\map_{f}(\ww)_i = \xi_i + 2(1 - 2\ell(f(Z_i),\tilde{Y}_i))$ as defined in Lemma~\ref{lem:adjacent}.
\end{lemma}

\begin{proof}
    Let $i\in[m]$ be any index and let $\mathbf{r}\in\rr^m$. Since $\xi_i\sim \lap(2m/\varepsilon)$, we know 

    \[
        \pp(\xi_i=r_i)=\frac{\varepsilon}{4m}\exp\left(\frac{-\vert r_i\vert\varepsilon}{2m}\right)
    \]

    For any $t,t'\in\rr$ such that $\vert t-t'\vert\leq 2$, we get 

    \[
        \pp(\xi_i=t')\leq \exp(\varepsilon/m)\pp(\xi_i=t)
    \]

    Since for all $i\in [d]$, $\vert \map_f(r)_i-r_i\vert\leq 2$, we get

    \[
        \frac{\pp(\ww=\map_f(\mathbf{r}))}{\pp(\ww=\mathbf{r})}=\prod_{i=1}^m\frac{\pp(\xi_i=\map_f(r)_i)}{\pp(\xi_i=r_i)}\leq \prod_{i=1}^m \exp(\varepsilon/m)=\exp(\varepsilon).
    \]
\end{proof}

Our proof of privacy also makes use of the following lemma, which says that minimizers are unique with probability 1~\citep[Lemma 4]{neel2019use}.

\begin{lemma}\label{lem:mzero}
    Let $\tilde{\cD}=(\tilde{\cD}_x,\tilde{\cD}_y)\in(\cX\times\{0,1\})^m$.  Consider $\proj{\cF}{\tilde{\cD}_x}$, where $\proj{\cF}{\tilde{\cD}_x}$ is the projection of $\cF$ on $\tilde{\cD}_x$. For every dataset $\cD$, there is a subset $B\subseteq \rr^m$ such that:
    \begin{itemize}
        \item $\pp(\ww\in B)=0$ and 
        \item On the restricted domain $\rr^m\setminus B$, there is a unique minimizer $\tilde{f}\in \arg\min_{f\in\proj{\cF}{\tilde{\cD}_x}}\cL_{\ww,\cD,\tilde{\cD}}(f)$.
    \end{itemize}
\end{lemma}

Using standard results about Laplace perturbations stated in Lemma~\ref{lem:laplace} and Lemma~\ref{lem:mzero} and using the perturbation coupling bound in Lemma~\ref{lem:adjacent}, we get the our formal statement of privacy for optimizing over the projected function class as follows:

\begin{lemma}(Privacy over Projection)\label{lem:projpriv}
Let $\cD,\cD'$ be arbitrary datasets containing $n$ points each. Let $\tilde{\cD}=(\tilde{\cD}_x,\tilde{\cD}_y)\in(\cX\times\{0,1\})^m$. Then,
\[
    \pp(\erm(\proj{\cF}{\tilde{\cD}_x},\cL_{\ww,\cD,\tilde{\cD}})=f)\leq e^{\varepsilon}\pp(\erm(\proj{\cF}{\tilde{\cD}_x},\cL_{\ww,\cD',\tilde{\cD}})=f).
\]
\end{lemma}

\begin{proof}
We calculate,
    \begin{align}
        \pp(\erm(\proj{\cF}{\tilde{\cD}_x},\cL_{\ww,\cD,\tilde{\cD}})=f)&=\pp(\ww\in\mathcal{E}(f,\cD,\tilde{\cD}))\\        
        &=\int_{\rr^m}\pp(\ww)\indi(\xi\in\mathcal{E}(f,\cD,\tilde{\cD}))d\ww\\ &=\int_{\rr^m\setminus B}\pp(\ww)\indi(\xi\in\mathcal{E}(f,\cD,\tilde{\cD}))d\ww&& \text{$B$ has 0 measure by Lemma \ref{lem:mzero}}\\
        &\leq \int_{\rr^m\setminus B}\pp(\ww)\indi(\map_{f}(\ww)\in\mathcal{E}(f,\cD',\tilde{\cD}))d\ww && \text{Lemma \ref{lem:adjacent}} \\
        &\leq \int_{\rr^m\setminus B}e^{\varepsilon}\pp(\map_f(\ww))\indi(\map_{f}(\ww)\in\mathcal{E}(f,\cD',\tilde{\cD}))d\ww && \text{Lemma \ref{lem:laplace}} \\
        &\leq\int_{\rr^m\setminus\map_f(B)}e^{\varepsilon}\pp(\ww)\indi(\ww\in\mathcal{E}(f,\cD',\tilde{\cD}))\frac{\partial\map_f}{\partial\ww}d\ww && \text{Change of variables } \ww \rightarrow \map_{f}(\ww)\\
        &= \int_{\rr^m}e^{\varepsilon}\pp(\ww)\indi(\ww\in\mathcal{E}(f,D',\tilde{\cD}))d\ww && \map_{f}(B) \text{ has 0 measure, } \left| \frac{\partial \map_{f}}{\partial \ww} \right| = 1 \\
        &=e^{\varepsilon}\pp(\ww\in\mathcal{E}(f,D',\tilde{\cD}))\\
        &=e^{\varepsilon}\pp(\erm(\proj{\cF}{\tilde{\cD}_x},\cL_{\ww,\cD',\tilde{\cD}})=f).
    \end{align}
\end{proof}

    \begin{theorem}\label{thm:privacy}
        Algorithm \ref{alg:RRSPM} is $\varepsilon$-pure differentially private.
    \end{theorem}

    \begin{proof}
        Let $\cD$ and $\cD'$ be any neighbouring datasets. Fix any function $\hat{f}\in\cF$. We now show that 
        \[
            \pp(\rrspm(\cD)=\hat{f})\leq e^{\varepsilon}\pp(\rrspm(\cD')=\hat{f}),
        \]

        where the probability is taken over the randomness of the algorithm. From the definition of the projection, we know that there exists a unique function in the projection say, $\tilde{f}\in \proj{\cF}{\tilde{\cD}_x}$ such that $\hat{f}(Z_i)=\tilde{f}(Z_i)$ for all $i\in[m]$. 

        Using Lemma \ref{lem:projpriv}, we know that 

        \[
              \pp(\erm(\proj{\cF}{\tilde{\cD}_x},\cL_{\ww,\cD,\tilde{\cD}})=\tilde{f})\leq e^{\varepsilon}\pp(\erm(\proj{\cF}{\tilde{\cD}_x},\cL_{\ww,\cD',\tilde{\cD},\tilde{\cD}_y})=\tilde{f}).
        \]

        As defined in the algorithm, let $\tilde{\cL}(f)=\sum_{i=1}^m\ell(f(Z_i), \tilde{f}(Z_i))$. From the definition of $\tilde{\cL}$, it follows that $\hat{f}\in \underset{f\in\cF}{\arg\min}\;\tilde{\cL}(f)$. Following the post-processing guarantee of differential privacy, it is easy to see that 

        \[
            \pp(\rrspm(\cD)=\hat{f})\leq e^{\varepsilon}\pp(\rrspm(\cD')=\hat{f}).
        \]
    \end{proof}
    
    \subsection{Accuracy Analysis}\label{app:classifcation_accuracy}
        
    In this section we analyze the accuracy of our algorithm. Let $f^*$ denote the function in the hypothesis class that minimizes the loss with respect to the distribution $\nu$ i.e. $f^*\in\underset{f\in\cF}{\arg\min}L_{\nu}(f)$. Let $\hat{f}$ denote the output hypothesis of our algorithm. We show that the loss of $\hat{f}$ is close to $f^*$ with respect to the data generating distribution $\nu$.

    As in the algorithm description, let $\tilde{f}$ be the function that minimizes the perturbed loss. Our algorithm outputs $\hat{f}$ whose labelling is consistent with $\tilde{f}$ on the public dataset $\tilde{\cD}_x$. Since $\tilde{\cD}_x$ is sampled from the base distribution, it follows from VC theorem that $\hat{f}$ and $\tilde{f}$ are close under the base distribution $\mu$. We show in Lemma~\ref{lem:smoothclose} that $\hat{f}$ and $\tilde{f}$ are close under $\nu$ by leveraging that $\nu_x$ is a $\sigma$-smooth distribution.
    Using standard results about Laplace perturbations, we show that $f'$ is close to $\tilde{f}$ in Lemma~\ref{lem:laperror}, where $f'$ is the empirical risk minimizer over $\cD$. Using the VC theorem, it is easy to see that $f'$ is close to $f^*$ and consequently using the triangle inequality we finish the proof by showing that $f^*$ and $\hat{f}$ are close under $\nu$.

    We now state the VC theorem below which we use in our analysis.

    \begin{theorem}\label{th:uc}
         Let $\cD=\{(X_1,Y_1),\dots ,(X_n,Y_n)\}$ where for all $i\in[n]$, $(X_i,Y_i)\in \cX\times\{0,1\}$ are sampled from a fixed distribution $\nu$. Let $L_D(f)=\frac{1}{n}\abs{\{i:f(X_i)\neq Y_i\}}$ and let $L_\nu(f)=\ee_{(X,Y)\sim \nu}[ \indi(f(X)\neq Y]$. If the function class $\cF$ has $VC$ dimension $d$ then,
    
        $$\pp\left(\underset{f\in\cF}{\sup}\vert L_{\cD}(f)-L_{\nu}(f)\vert \leq O\left(\sqrt{\frac{d+\log(1/\beta)}{m}}\right)\right)\geq 1-\beta .$$
    
        In particular if $f^*\in \underset{f\in\cF}{\arg\min}\;L_{\cD}(f)$ then, 
    
        $$\pp\left(\vert L_{\nu}(f^*)-\underset{f\in\cF}{\arg\min}L_{\nu}(f)\vert \leq O\left(\sqrt{\frac{d+\log(1/\beta)}{m}}\right)\right)\geq 1-\beta .$$       
    \end{theorem}

    \begin{lemma}\label{lem:smoothclose}
         Let $\nu_x$ be a $\sigma$-smooth distribution that is $\norm{\frac{d\nu_x}{d\mu}}\leq \frac{1}{\sigma}$. Let $L_{\nu}$ be the loss function as defined in Defintion~\ref{def:pac}. Let $\hat{f}$ be the hypothesis returned by our algorithm and let $\tilde{f}\in\erm(\cF,\cL_{\ww,\cD,\tilde{\cD}})$. Then,  
         \[
            L_{\nu}(\hat{f})-L_{\nu}(\tilde{f})\leq O\left(\sqrt{\frac{d+\log(1/\beta)}{m\sigma^2}}\right),
            \]
        with probability $1-\beta.$
    \end{lemma}

    \begin{proof}
        Let $\alpha' = O\left(\sqrt{\frac{d+\log(1/\beta)}{m}}\right)$. First we show that $\vert\underset{\substack{{X}\sim\mu\\
            Y=\tilde{f}(X)}}{L(\hat{f})}-\underset{\substack{{X}\sim\mu\\
            Y=\tilde{f}(X)}}{L(\tilde{f})}\vert\leq 2\alpha'$ with probability at least $1-\beta$. Consider the dataset $\hat{D}=\{(X_1,\tilde{f}(X_1),\dots ,(X_m,\tilde{f}(X_m))\}$ we minimize over and output $\hat{f}$. 
        Using Theorem \ref{th:uc} we know that, 
        \[
            \pp\left(\vert\underset{\substack{{X}\sim\mu\\
            Y=\tilde{f}(X)}}{L(\hat{f})}-\underset{\substack{{X}\sim\mu\\
            Y=\tilde{f}(X)}}{L(\tilde{f})}\vert\leq 2\alpha'+ \vert L_{\hat{D}}(\hat{f})-L_{\hat{D}}(\tilde{f})\vert\right)\geq 1-\beta.
        \]
        Since $\underset{\substack{{X}\sim\mu\\
            Y=\tilde{f}(X)}}{L(\tilde{f})}=0$ and $\vert L_{\tilde{\cD}}(\hat{f})-L_{\tilde{\cD}}(\tilde{f})\vert=0$ and  we get that with probability at least $1-\beta$,

        \begin{equation}\label{eq:muuc}
            \pp_{X\sim\mu}(\hat{f}(x)\neq\tilde{f}(x))\leq 2\alpha'.
        \end{equation}
        We now show that $\pp_{X\sim\nu_x}(\hat{f}(x)\neq\tilde{f}(x))\leq 2\alpha'/\sigma$ with probability at least $1-\beta$.
        \begin{align}
            \pp_{X\sim\nu_x}(\hat{f}(x)\neq\tilde{f}(x))&=\ee_X\sim\nu_x[\indi(\hat{f}(x)\neq\tilde{f}(x))]\\
            &=\sum_{X\in\cX}\pp(\nu_x=X)\indi(\hat{f}(x)\neq\tilde{f}(x))\\
            &\leq \frac{1}{\sigma}\sum_{X\in\cX}\pp(\mu=X)\indi(\hat{f}(x)\neq\tilde{f}(x))&& \text{Since }\norm{\frac{d\nu_x}{d\mu}}\leq \frac{1}{\sigma}\\
            &= \frac{1}{\sigma}\left(\ee_X\sim\mu[\indi(\hat{f}(x)\neq\tilde{f}(x))]\right)\\
            &= \frac{1}{\sigma}\left(\pp_{X\sim\mu}(\hat{f}(x)\neq\tilde{f}(x))\right)\\
            &\leq \frac{2\alpha}{\sigma}&& \text{Using Eq \ref{eq:muuc}}
        \end{align}
        Having shown that with probability at least $1-\beta$,
        \begin{equation}\label{eq:smoothclose}
            \pp_{X\sim\nu_x}(\hat{f}(x)\neq\tilde{f}(x))\leq 2\alpha'/\sigma,
        \end{equation}
         we now prove that with probability at least $1-\beta$,
        \[
            \vert\underset{\substack{{X}\sim\nu_x\\
            Y\sim\nu_y\mid X}}{L(\hat{f})}-\underset{\substack{{X}\sim\nu_x\\
            Y\sim\nu_y\mid X}}{L(\tilde{f})}\vert\leq 2\alpha'/\sigma.
        \]
        which is equivalent to the theorem statement. 
        Using the triangle inequality we get
        \begin{align}
            \vert\underset{\substack{{X}\sim\nu_x\\
            Y\sim\nu_y\mid X}}{L(\hat{f})}-\underset{\substack{{X}\sim\nu_x\\
            Y\sim\nu_y\mid X}}{L(\tilde{f})}\vert &\leq \ee_{X\sim\nu_x}\left[\left\vert\pp_{Y\sim\nu_{y}\mid X}(\hat{f}(X)\neq Y)-\pp_{Y\sim\nu_{y}\mid X}(\tilde{f}(X)\neq Y)\right\vert\right]\\
            &=\ee_{X\sim\nu_x}[\indi(\tilde{f}(X)\neq\hat{f}(X))]\\
            &= 2\alpha'/\sigma
        \end{align}
        where the second equation follows from the observation that for any value of $X$, if $\hat{f}(X)=\tilde{f}(X)$ then the difference of the probabilities equate to $0$ and if $\hat{f}(X)\neq\tilde{f}(X)$ then the difference of the probabilities equate to $1$ and the last equation follows from Equation \ref{eq:smoothclose}.

        Substituting the value of $\alpha'$ proves that         with probability $1-\beta$,       
        \[
            L_{\nu}(\hat{f})-L_{\nu}(\tilde{f})\leq O\left(\sqrt{\frac{d+\log(1/\beta)}{m\sigma^2}}\right).
            \]
    \end{proof}

    \begin{lemma}\label{lem:laperror}
         Let $\cF$ be a function class with $VC$ dimension $d$. Let $\cD=\{(X_1,Y_1),\dots ,(X_n,Y_n)\}$ and $\tilde{\cD}_x=\{Z_1,\dots ,Z_m\}$ where $Z_i\sim\mu$. Let $L_{\cD}$ be the loss function as defined in Definion~\ref{def:pac}. Let $f'\in\erm(\cF,\cL_{\cD})$ and let $\tilde{f}\in\erm(\cF,\cL_{\ww,\cD,\tilde{\cD}})$ then,
         \[
         L_{\cD}(\tilde{f})-L_{\cD}(f')\leq \frac{4m^2\log(m/\beta)}{\varepsilon n},
         \]
         with probability $1-\beta$.
    \end{lemma}

  \begin{proof}
        Following the algorithm we know that for all $i\in [m],\xi_i\sim \lap(2m/\varepsilon)$. Using Chernoff's bound and a union bound we get that with probability $1-\beta$,
        \[
            \forall i\in[m],\vert \xi_i\vert \leq \frac{2m\log (m/\beta)}{\epsilon}.
        \]
        Since $\tilde{f}(Z_i)\in\{0,1\}$ for all $Z_i\in\tilde{\cD}$, with probability $1-\beta$, $\cL_{\ww,\cD,\tilde{\cD}}(\tilde{f})\geq \cL_{\cD}(\tilde{f})-m\cdot \frac{2m\log (m/\beta)}{\epsilon}$. Similarly, $\cL_{\ww,\cD,\tilde{\cD}}({f'})\leq \cL_{\cD}({f'})+m\cdot \frac{2m\log (m/\beta)}{\epsilon}$.
        Dividing by $n$ and combining the bounds we get,
        \[
            L_{\cD}(\tilde{f})-L_{\cD}(f')\leq \frac{4m^2\log(m/\beta)}{\varepsilon n},
        \]
        with probability $1-\beta.$
    \end{proof}

        \begin{theorem}\label{thm:bcpac}
        Let $\hat{f}$ be as defined in Algorithm with $\cD$ sampled from some distribution $\nu$ such that $\norm{\frac{d\nu_x}{d\mu}}\leq \frac{1}{\sigma}$. Suppose the function class $\cF$ has VC-dimension $d$ then setting 
        \[
            m= O\left(\frac{d+\log(1/\beta)}{\alpha^2\sigma^2}\right)
        \]
        yields an $\varepsilon$- pure differentially private $(\alpha,\beta)$-learner as long as 
        \[
            n\geq  \tilde{\Omega}\left(\frac{d^2\log(1/\beta)}{\alpha^5\sigma^4\epsilon}\right),
        \]
        where $\tilde{\Omega}$ hides $\log$ factors.
    \end{theorem}

    \begin{proof}
        Let $\nu$ be the distribution from which the given dataset $\cD$ is sampled where $\norm{\frac{d\nu_x}{d\mu}}\leq \frac{1}{\sigma}$. Let $\hat{f}$ be the output of our algorithm, $f'\in\underset{f\in\cF}{\arg\min}\cL_{\cD}(f)$, $f^*$ be the target function from our function class $\cF$ and $\tilde{f}\in\underset{f\in\cF}{\arg\min}\cL_{\ww,\cD,\tilde{\cD}}(f)$. We wish to compare the guarantee of the function $\hat{f}$  with respect to the function $f^*$ with respect to the distribution $\nu$ i.e. $\vert L_{\nu}(\hat{h})-L_{\nu}(f^*)\vert$.

        In our analysis below, we break the loss function several times using triangle inequality and bound each term using previously stated results via a union bound.
        \begin{align}
            &\vert L_{\nu}(\hat{f})-L_{\nu}(f^*)\vert\\&\leq\vert L_{\nu}(\hat{f})-L_{\nu}(\tilde{f})\vert+\vert L_{\nu}(\tilde{f})-L_{\cD}(\tilde{f})\vert+\vert L_{\cD}(\tilde{f})-L_{\cD}(f')\vert+\vert L_{\cD}(f')-L_{\nu}(f^*)\vert\\
            &\leq \underbrace{O\left(\sqrt{\frac{d+\log(1/\beta)}{\sigma^2m}}\right)}_{\text{Lemma } \ref{lem:smoothclose}}+\underbrace{O\left(\sqrt{\frac{d+\log(1/\beta)}{n}}\right)}_{\text{Theorem }\ref{th:uc}}+\underbrace{O\left(\frac{m^2\log(m/\beta)}{\varepsilon n}\right)}_{\text{Lemma }\ref{lem:laperror}}+\underbrace{O\left(\sqrt{\frac{d+\log(1/\beta)}{n}}\right)}_{\text{Theorem }\ref{th:uc}}\\
            &=O\left(\sqrt{\frac{d+\log(1/\beta)}{\sigma^2m}}+\sqrt{\frac{d+\log(1/\beta)}{n}}+\frac{m^2\log(m/\beta)}{\varepsilon n}\right).
        \end{align}

        Setting         
        \[
        m=O\left(\frac{d+\log(1/\beta)}{\alpha^2\sigma^2}\right),\; n= O\left(\frac{d^2\log d\cdot \log(1/\beta)}{\alpha^5\sigma^4\epsilon}\right)
        \]
        and combining the guarantee from Theorem \ref{thm:privacy} we get that $\rrspm$ (Algorithm \ref{alg:RRSPM}) is a $\varepsilon$- pure differentially private $(\alpha,\beta)$-learner.
    \end{proof}

%% file: arxiv.bbl
\begin{thebibliography}{72}
\providecommand{\natexlab}[1]{#1}
\providecommand{\url}[1]{\texttt{#1}}
\expandafter\ifx\csname urlstyle\endcsname\relax
  \providecommand{\doi}[1]{doi: #1}\else
  \providecommand{\doi}{doi: \begingroup \urlstyle{rm}\Url}\fi

\bibitem[Abernethy et~al.(2019)Abernethy, Jung, Lee, McMillan, and Tewari]{abernethy2019online}
Jacob~D Abernethy, Young~Hun Jung, Chansoo Lee, Audra McMillan, and Ambuj Tewari.
\newblock Online learning via the differential privacy lens.
\newblock \emph{Advances in Neural Information Processing Systems}, 32, 2019.

\bibitem[Alon et~al.(2019)Alon, Livni, Malliaris, and Moran]{AlonLMM19}
Noga Alon, Roi Livni, Maryanthe Malliaris, and Shay Moran.
\newblock Private {PAC} learning implies finite littlestone dimension.
\newblock In Moses Charikar and Edith Cohen, editors, \emph{Proceedings of the 51st Annual {ACM} {SIGACT} Symposium on Theory of Computing, {STOC} 2019, Phoenix, AZ, USA, June 23-26, 2019}, pages 852--860. {ACM}, 2019.
\newblock \doi{10.1145/3313276.3316312}.
\newblock URL \url{https://doi.org/10.1145/3313276.3316312}.

\bibitem[Amid et~al.(2022)Amid, Ganesh, Mathews, Ramaswamy, Song, Steinke, Suriyakumar, Thakkar, and Thakurta]{amid2022public}
Ehsan Amid, Arun Ganesh, Rajiv Mathews, Swaroop Ramaswamy, Shuang Song, Thomas Steinke, Vinith~M Suriyakumar, Om~Thakkar, and Abhradeep Thakurta.
\newblock Public data-assisted mirror descent for private model training.
\newblock In \emph{International Conference on Machine Learning}, pages 517--535. PMLR, 2022.

\bibitem[Bartlett et~al.(1994)Bartlett, Long, and Williamson]{bartlett1994fat}
Peter~L Bartlett, Philip~M Long, and Robert~C Williamson.
\newblock Fat-shattering and the learnability of real-valued functions.
\newblock In \emph{Proceedings of the seventh annual conference on Computational learning theory}, pages 299--310, 1994.

\bibitem[Bassily et~al.(2019)Bassily, Moran, and Alon]{BassilyMA19}
Raef Bassily, Shay Moran, and Noga Alon.
\newblock Limits of private learning with access to public data.
\newblock In Hanna~M. Wallach, Hugo Larochelle, Alina Beygelzimer, Florence d'Alch{\'{e}}{-}Buc, Emily~B. Fox, and Roman Garnett, editors, \emph{Advances in Neural Information Processing Systems 32: Annual Conference on Neural Information Processing Systems 2019, NeurIPS 2019, December 8-14, 2019, Vancouver, BC, Canada}, pages 10342--10352, 2019.
\newblock URL \url{https://proceedings.neurips.cc/paper/2019/hash/9a6a1aaafe73c572b7374828b03a1881-Abstract.html}.

\bibitem[Bassily et~al.(2020{\natexlab{a}})Bassily, Cheu, Moran, Nikolov, Ullman, and Wu]{BassilyCMNUW20}
Raef Bassily, Albert Cheu, Shay Moran, Aleksandar Nikolov, Jonathan~R. Ullman, and Zhiwei~Steven Wu.
\newblock Private query release assisted by public data.
\newblock In \emph{Proceedings of the 37th International Conference on Machine Learning, {ICML} 2020, 13-18 July 2020, Virtual Event}, volume 119 of \emph{Proceedings of Machine Learning Research}, pages 695--703. {PMLR}, 2020{\natexlab{a}}.
\newblock URL \url{http://proceedings.mlr.press/v119/bassily20a.html}.

\bibitem[Bassily et~al.(2020{\natexlab{b}})Bassily, Moran, and Nandi]{bassily2020learning}
Raef Bassily, Shay Moran, and Anupama Nandi.
\newblock Learning from mixtures of private and public populations.
\newblock \emph{Advances in Neural Information Processing Systems}, 33:\penalty0 2947--2957, 2020{\natexlab{b}}.

\bibitem[Bassily et~al.(2022)Bassily, Mohri, and Suresh]{bassily2022private}
Raef Bassily, Mehryar Mohri, and Ananda~Theertha Suresh.
\newblock Private domain adaptation from a public source.
\newblock \emph{arXiv preprint arXiv:2208.06135}, 2022.

\bibitem[Bassily et~al.(2023)Bassily, Mohri, and Suresh]{bassily2023principled}
Raef Bassily, Mehryar Mohri, and Ananda~Theertha Suresh.
\newblock Principled approaches for private adaptation from a public source.
\newblock In \emph{International Conference on Artificial Intelligence and Statistics}, pages 8405--8432. PMLR, 2023.

\bibitem[Beimel et~al.(2014)Beimel, Nissim, and Stemmer]{beimel2014learning}
Amos Beimel, Kobbi Nissim, and Uri Stemmer.
\newblock Learning privately with labeled and unlabeled examples.
\newblock In \emph{Proceedings of the twenty-sixth annual ACM-SIAM symposium on Discrete algorithms}, pages 461--477. SIAM, 2014.

\bibitem[Ben-David et~al.(2023)Ben-David, Bie, Canonne, Kamath, and Singhal]{ben2023private}
Shai Ben-David, Alex Bie, Cl{\'e}ment~L Canonne, Gautam Kamath, and Vikrant Singhal.
\newblock Private distribution learning with public data: The view from sample compression.
\newblock \emph{arXiv preprint arXiv:2308.06239}, 2023.

\bibitem[Bie et~al.(2022)Bie, Kamath, and Singhal]{bie2022private}
Alex Bie, Gautam Kamath, and Vikrant Singhal.
\newblock Private estimation with public data.
\newblock \emph{Advances in Neural Information Processing Systems}, 35:\penalty0 18653--18666, 2022.

\bibitem[Block and Polyanskiy(2023)]{block2023sample}
Adam Block and Yury Polyanskiy.
\newblock The sample complexity of approximate rejection sampling with applications to smoothed online learning.
\newblock In Gergely Neu and Lorenzo Rosasco, editors, \emph{Proceedings of Thirty Sixth Conference on Learning Theory}, volume 195 of \emph{Proceedings of Machine Learning Research}, pages 228--273. PMLR, 12--15 Jul 2023.
\newblock URL \url{https://proceedings.mlr.press/v195/block23a.html}.

\bibitem[Block and Simchowitz(2022)]{block2022efficient}
Adam Block and Max Simchowitz.
\newblock Efficient and near-optimal smoothed online learning for generalized linear functions.
\newblock \emph{Advances in Neural Information Processing Systems}, 35:\penalty0 7477--7489, 2022.

\bibitem[Block et~al.(2022)Block, Dagan, Golowich, and Rakhlin]{block2022smoothed}
Adam Block, Yuval Dagan, Noah Golowich, and Alexander Rakhlin.
\newblock Smoothed online learning is as easy as statistical learning.
\newblock In \emph{Conference on Learning Theory}, pages 1716--1786. PMLR, 2022.

\bibitem[Block et~al.(2023{\natexlab{a}})Block, Simchowitz, and Rakhlin]{block2023oracle}
Adam Block, Max Simchowitz, and Alexander Rakhlin.
\newblock Oracle-efficient smoothed online learning for piecewise continuous decision making.
\newblock In Gergely Neu and Lorenzo Rosasco, editors, \emph{Proceedings of Thirty Sixth Conference on Learning Theory}, volume 195 of \emph{Proceedings of Machine Learning Research}, pages 1618--1665. PMLR, 12--15 Jul 2023{\natexlab{a}}.
\newblock URL \url{https://proceedings.mlr.press/v195/block23b.html}.

\bibitem[Block et~al.(2023{\natexlab{b}})Block, Simchowitz, and Tedrake]{block2023smoothed}
Adam Block, Max Simchowitz, and Russ Tedrake.
\newblock Smoothed online learning for prediction in piecewise affine systems.
\newblock In \emph{Advances in Neural Information Processing Systems}. Curran Associates, Inc., 2023{\natexlab{b}}.
\newblock URL \url{https://openreview.net/pdf?id=Izt7rDD7jN}.

\bibitem[Blum and Rivest(1988)]{blum1988training}
Avrim Blum and Ronald Rivest.
\newblock Training a 3-node neural network is np-complete.
\newblock \emph{Advances in neural information processing systems}, 1, 1988.

\bibitem[Bousquet(2002)]{bousquet2002concentration}
Olivier Bousquet.
\newblock \emph{Concentration inequalities and empirical processes theory applied to the analysis of learning algorithms}.
\newblock PhD thesis, {\'E}cole Polytechnique: Department of Applied Mathematics Paris, France, 2002.

\bibitem[Bun et~al.(2015)Bun, Nissim, Stemmer, and Vadhan]{bun2015differentially}
Mark Bun, Kobbi Nissim, Uri Stemmer, and Salil Vadhan.
\newblock Differentially private release and learning of threshold functions.
\newblock In \emph{2015 IEEE 56th Annual Symposium on Foundations of Computer Science}, pages 634--649. IEEE, 2015.

\bibitem[Bun et~al.(2020)Bun, Livni, and Moran]{bun2020equivalence}
Mark Bun, Roi Livni, and Shay Moran.
\newblock An equivalence between private classification and online prediction.
\newblock In \emph{2020 IEEE 61st Annual Symposium on Foundations of Computer Science (FOCS)}, pages 389--402. IEEE, 2020.

\bibitem[Cesa-Bianchi and Lugosi(2006)]{cesa2006prediction}
Nicolo Cesa-Bianchi and G{\'a}bor Lugosi.
\newblock \emph{Prediction, learning, and games}.
\newblock Cambridge university press, 2006.

\bibitem[Cesa-Bianchi et~al.(2023)Cesa-Bianchi, Cesari, Colomboni, Fusco, and Leonardi]{cesa2023repeated}
Nicol{\`o} Cesa-Bianchi, Tommaso~R Cesari, Roberto Colomboni, Federico Fusco, and Stefano Leonardi.
\newblock Repeated bilateral trade against a smoothed adversary.
\newblock In \emph{The Thirty Sixth Annual Conference on Learning Theory}, pages 1095--1130. PMLR, 2023.

\bibitem[Chaudhuri et~al.(2011)Chaudhuri, Monteleoni, and Sarwate]{chaudhuri2011differentially}
Kamalika Chaudhuri, Claire Monteleoni, and Anand~D Sarwate.
\newblock Differentially private empirical risk minimization.
\newblock \emph{Journal of Machine Learning Research}, 12\penalty0 (3), 2011.

\bibitem[Chernozhukov et~al.(2015)Chernozhukov, Chetverikov, and Kato]{chernozhukov2015comparison}
Victor Chernozhukov, Denis Chetverikov, and Kengo Kato.
\newblock Comparison and anti-concentration bounds for maxima of gaussian random vectors.
\newblock \emph{Probability Theory and Related Fields}, 162:\penalty0 47--70, 2015.

\bibitem[Dud{\'\i}k et~al.(2020)Dud{\'\i}k, Haghtalab, Luo, Schapire, Syrgkanis, and Vaughan]{dudik2020oracle}
Miroslav Dud{\'\i}k, Nika Haghtalab, Haipeng Luo, Robert~E Schapire, Vasilis Syrgkanis, and Jennifer~Wortman Vaughan.
\newblock Oracle-efficient online learning and auction design.
\newblock \emph{Journal of the ACM (JACM)}, 67\penalty0 (5):\penalty0 1--57, 2020.

\bibitem[Dudley(1969)]{dudley1969speed}
Richard~Mansfield Dudley.
\newblock The speed of mean glivenko-cantelli convergence.
\newblock \emph{The Annals of Mathematical Statistics}, 40\penalty0 (1):\penalty0 40--50, 1969.

\bibitem[Durvasula et~al.(2023)Durvasula, Haghtalab, and Zampetakis]{durvasula2023smoothed}
Naveen Durvasula, Nika Haghtalab, and Manolis Zampetakis.
\newblock Smoothed analysis of online non-parametric auctions.
\newblock In \emph{Proceedings of the 24th ACM Conference on Economics and Computation}, pages 540--560, 2023.

\bibitem[Dwork et~al.(2006)Dwork, McSherry, Nissim, and Smith]{DworkMNS06}
Cynthia Dwork, Frank McSherry, Kobbi Nissim, and Adam~D. Smith.
\newblock Calibrating noise to sensitivity in private data analysis.
\newblock In Shai Halevi and Tal Rabin, editors, \emph{Theory of Cryptography, Third Theory of Cryptography Conference, {TCC} 2006, New York, NY, USA, March 4-7, 2006, Proceedings}, volume 3876 of \emph{Lecture Notes in Computer Science}, pages 265--284. Springer, 2006.
\newblock \doi{10.1007/11681878\_14}.
\newblock URL \url{https://doi.org/10.1007/11681878\_14}.

\bibitem[Dwork et~al.(2014)Dwork, Roth, et~al.]{dwork2014algorithmic}
Cynthia Dwork, Aaron Roth, et~al.
\newblock The algorithmic foundations of differential privacy.
\newblock \emph{Foundations and Trends{\textregistered} in Theoretical Computer Science}, 9\penalty0 (3--4):\penalty0 211--407, 2014.

\bibitem[Foster and Rakhlin(2020)]{foster2020beyond}
Dylan Foster and Alexander Rakhlin.
\newblock Beyond ucb: Optimal and efficient contextual bandits with regression oracles.
\newblock In \emph{International Conference on Machine Learning}, pages 3199--3210. PMLR, 2020.

\bibitem[Foster et~al.(2021)Foster, Kakade, Qian, and Rakhlin]{foster2021statistical}
Dylan~J Foster, Sham~M Kakade, Jian Qian, and Alexander Rakhlin.
\newblock The statistical complexity of interactive decision making.
\newblock \emph{arXiv preprint arXiv:2112.13487}, 2021.

\bibitem[Gaboardi et~al.(2014)Gaboardi, Arias, Hsu, Roth, and Wu]{gaboardi2014dual}
Marco Gaboardi, Emilio Jes{\'u}s~Gallego Arias, Justin Hsu, Aaron Roth, and Zhiwei~Steven Wu.
\newblock Dual query: Practical private query release for high dimensional data.
\newblock In \emph{International Conference on Machine Learning}, pages 1170--1178. PMLR, 2014.

\bibitem[Ghazi et~al.(2021)Ghazi, Golowich, Kumar, and Manurangsi]{ghazi2021sample}
Badih Ghazi, Noah Golowich, Ravi Kumar, and Pasin Manurangsi.
\newblock Sample-efficient proper pac learning with approximate differential privacy.
\newblock In \emph{Proceedings of the 53rd Annual ACM SIGACT Symposium on Theory of Computing}, pages 183--196, 2021.

\bibitem[Golatkar et~al.(2022)Golatkar, Achille, Wang, Roth, Kearns, and Soatto]{Golatkar2022}
Aditya Golatkar, Alessandro Achille, Yu-Xiang Wang, Aaron Roth, Michael Kearns, and Stefano Soatto.
\newblock Mixed differential privacy in computer vision.
\newblock In \emph{CVPR 2022}, 2022.
\newblock URL \url{https://www.amazon.science/publications/mixed-differential-privacy-in-computer-vision}.

\bibitem[Goldman et~al.(1993)Goldman, Kearns, and Schapire]{goldman1993exact}
Sally~A Goldman, Michael~J Kearns, and Robert~E Schapire.
\newblock Exact identification of read-once formulas using fixed points of amplification functions.
\newblock \emph{SIAM Journal on Computing}, 22\penalty0 (4):\penalty0 705--726, 1993.

\bibitem[Gordon(1999)]{gordon1999regret}
Geoffrey~J Gordon.
\newblock Regret bounds for prediction problems.
\newblock In \emph{Proceedings of the twelfth annual conference on Computational learning theory}, pages 29--40, 1999.

\bibitem[Haghtalab et~al.(2020)Haghtalab, Roughgarden, and Shetty]{haghtalab2020smoothed}
Nika Haghtalab, Tim Roughgarden, and Abhishek Shetty.
\newblock Smoothed analysis of online and differentially private learning.
\newblock \emph{Advances in Neural Information Processing Systems}, 33:\penalty0 9203--9215, 2020.

\bibitem[Haghtalab et~al.(2022{\natexlab{a}})Haghtalab, Han, Shetty, and Yang]{haghtalab2022oracle}
Nika Haghtalab, Yanjun Han, Abhishek Shetty, and Kunhe Yang.
\newblock Oracle-efficient online learning for smoothed adversaries.
\newblock In S.~Koyejo, S.~Mohamed, A.~Agarwal, D.~Belgrave, K.~Cho, and A.~Oh, editors, \emph{Advances in Neural Information Processing Systems}, volume~35, pages 4072--4084. Curran Associates, Inc., 2022{\natexlab{a}}.
\newblock URL \url{https://proceedings.neurips.cc/paper_files/paper/2022/file/1a04df6a405210aab4986994b873db9b-Paper-Conference.pdf}.

\bibitem[Haghtalab et~al.(2022{\natexlab{b}})Haghtalab, Roughgarden, and Shetty]{haghtalab2022smoothed}
Nika Haghtalab, Tim Roughgarden, and Abhishek Shetty.
\newblock Smoothed analysis with adaptive adversaries.
\newblock In \emph{2021 IEEE 62nd Annual Symposium on Foundations of Computer Science (FOCS)}, pages 942--953. IEEE, 2022{\natexlab{b}}.

\bibitem[Hazan and Koren(2016)]{hazan2016computational}
Elad Hazan and Tomer Koren.
\newblock The computational power of optimization in online learning.
\newblock In \emph{Proceedings of the forty-eighth annual ACM symposium on Theory of Computing}, pages 128--141, 2016.

\bibitem[Kairouz et~al.(2021)Kairouz, Diaz, Rush, and Thakurta]{kairouz2021nearly}
Peter Kairouz, Monica~Ribero Diaz, Keith Rush, and Abhradeep Thakurta.
\newblock ({N}early) dimension independent private erm with adagrad rates via publicly estimated subspaces.
\newblock In \emph{Conference on Learning Theory}, pages 2717--2746. PMLR, 2021.

\bibitem[Kalai and Vempala(2005)]{kalai2005efficient}
Adam Kalai and Santosh Vempala.
\newblock Efficient algorithms for online decision problems.
\newblock \emph{Journal of Computer and System Sciences}, 71\penalty0 (3):\penalty0 291--307, 2005.

\bibitem[Kim and Pollard(1990)]{kim1990cube}
Jeankyung Kim and David Pollard.
\newblock Cube root asymptotics.
\newblock \emph{The Annals of Statistics}, pages 191--219, 1990.

\bibitem[Kozachinskiy and Steifer(2023)]{kozachinskiy2023simple}
Alexander Kozachinskiy and Tomasz Steifer.
\newblock Simple online learning with consistency oracle.
\newblock \emph{arXiv preprint arXiv:2308.08055}, 2023.

\bibitem[Le~Gall(2016)]{le2016brownian}
Jean-Fran{\c{c}}ois Le~Gall.
\newblock \emph{Brownian motion, martingales, and stochastic calculus}.
\newblock Springer, 2016.

\bibitem[Liu et~al.(2021{\natexlab{a}})Liu, Vietri, Steinke, Ullman, and Wu]{liu2021leveraging}
Terrance Liu, Giuseppe Vietri, Thomas Steinke, Jonathan Ullman, and Steven Wu.
\newblock Leveraging public data for practical private query release.
\newblock In \emph{International Conference on Machine Learning}, pages 6968--6977. PMLR, 2021{\natexlab{a}}.

\bibitem[Liu et~al.(2021{\natexlab{b}})Liu, Vietri, and Wu]{LiuVW21}
Terrance Liu, Giuseppe Vietri, and Steven Wu.
\newblock Iterative methods for private synthetic data: Unifying framework and new methods.
\newblock In Marc'Aurelio Ranzato, Alina Beygelzimer, Yann~N. Dauphin, Percy Liang, and Jennifer~Wortman Vaughan, editors, \emph{Advances in Neural Information Processing Systems 34: Annual Conference on Neural Information Processing Systems 2021, NeurIPS 2021, December 6-14, 2021, virtual}, pages 690--702, 2021{\natexlab{b}}.
\newblock URL \url{https://proceedings.neurips.cc/paper/2021/hash/0678c572b0d5597d2d4a6b5bd135754c-Abstract.html}.

\bibitem[Lowy et~al.(2023)Lowy, Li, Huang, and Razaviyayn]{lowy2023optimal}
Andrew Lowy, Zeman Li, Tianjian Huang, and Meisam Razaviyayn.
\newblock Optimal differentially private learning with public data.
\newblock \emph{arXiv preprint arXiv:2306.15056}, 2023.

\bibitem[McSherry and Talwar(2007)]{McSherryT07}
Frank McSherry and Kunal Talwar.
\newblock Mechanism design via differential privacy.
\newblock In \emph{48th Annual {IEEE} Symposium on Foundations of Computer Science {(FOCS} 2007), October 20-23, 2007, Providence, RI, USA, Proceedings}, pages 94--103. {IEEE} Computer Society, 2007.
\newblock \doi{10.1109/FOCS.2007.41}.
\newblock URL \url{https://doi.org/10.1109/FOCS.2007.41}.

\bibitem[Mendelson and Vershynin(2003)]{mendelson2003entropy}
Shahar Mendelson and Roman Vershynin.
\newblock Entropy and the combinatorial dimension.
\newblock \emph{Inventiones mathematicae}, 152\penalty0 (1):\penalty0 37--55, 2003.

\bibitem[Mhammedi et~al.(2023{\natexlab{a}})Mhammedi, Block, Foster, and Rakhlin]{mhammedi2023efficient}
Zakaria Mhammedi, Adam Block, Dylan~J Foster, and Alexander Rakhlin.
\newblock Efficient model-free exploration in low-rank mdps.
\newblock \emph{arXiv preprint arXiv:2307.03997}, 2023{\natexlab{a}}.

\bibitem[Mhammedi et~al.(2023{\natexlab{b}})Mhammedi, Foster, and Rakhlin]{mhammedi2023representation}
Zakaria Mhammedi, Dylan~J Foster, and Alexander Rakhlin.
\newblock Representation learning with multi-step inverse kinematics: An efficient and optimal approach to rich-observation rl.
\newblock \emph{arXiv preprint arXiv:2304.05889}, 2023{\natexlab{b}}.

\bibitem[Neel et~al.(2020)Neel, Roth, Vietri, and Wu]{Neel0VW20}
Seth Neel, Aaron Roth, Giuseppe Vietri, and Zhiwei~Steven Wu.
\newblock Oracle efficient private non-convex optimization.
\newblock In \emph{Proceedings of the 37th International Conference on Machine Learning, {ICML} 2020, 13-18 July 2020, Virtual Event}, volume 119 of \emph{Proceedings of Machine Learning Research}, pages 7243--7252. {PMLR}, 2020.
\newblock URL \url{http://proceedings.mlr.press/v119/neel20a.html}.

\bibitem[Neel et~al.(2019)Neel, Roth, and Wu]{neel2019use}
Seth~V Neel, Aaron~L Roth, and Zhiwei~Steven Wu.
\newblock How to use heuristics for differential privacy.
\newblock In \emph{2019 IEEE 60th Annual Symposium on Foundations of Computer Science (FOCS)}, pages 72--93. IEEE, 2019.

\bibitem[Nikolov et~al.(2013)Nikolov, Talwar, and Zhang]{nikolov2013geometry}
Aleksandar Nikolov, Kunal Talwar, and Li~Zhang.
\newblock The geometry of differential privacy: the sparse and approximate cases.
\newblock In \emph{Proceedings of the forty-fifth annual ACM symposium on Theory of computing}, pages 351--360, 2013.

\bibitem[Papernot et~al.(2018)Papernot, Song, Mironov, Raghunathan, Talwar, and Erlingsson]{PapernotSMRTE18}
Nicolas Papernot, Shuang Song, Ilya Mironov, Ananth Raghunathan, Kunal Talwar, and {\'{U}}lfar Erlingsson.
\newblock Scalable private learning with {PATE}.
\newblock In \emph{6th International Conference on Learning Representations, {ICLR} 2018, Vancouver, BC, Canada, April 30 - May 3, 2018, Conference Track Proceedings}. OpenReview.net, 2018.
\newblock URL \url{https://openreview.net/forum?id=rkZB1XbRZ}.

\bibitem[Polyanskiy and Wu(2022+)]{polyanskiy2022}
Yury Polyanskiy and Yihong Wu.
\newblock \emph{Information Theory: From Coding to Learning}.
\newblock Cambridge University Press, 2022+.

\bibitem[Rakhlin et~al.(2011)Rakhlin, Sridharan, and Tewari]{rakhlin2011online}
Alexander Rakhlin, Karthik Sridharan, and Ambuj Tewari.
\newblock Online learning: Stochastic, constrained, and smoothed adversaries.
\newblock \emph{Advances in neural information processing systems}, 24, 2011.

\bibitem[Rakhlin et~al.(2015)Rakhlin, Sridharan, and Tewari]{rakhlin2015online}
Alexander Rakhlin, Karthik Sridharan, and Ambuj Tewari.
\newblock Online learning via sequential complexities.
\newblock \emph{J. Mach. Learn. Res.}, 16\penalty0 (1):\penalty0 155--186, 2015.

\bibitem[Rakhlin et~al.(2017)Rakhlin, Sridharan, and Tsybakov]{rakhlin2017empirical}
Alexander Rakhlin, Karthik Sridharan, and Alexandre~B Tsybakov.
\newblock Empirical entropy, minimax regret and minimax risk.
\newblock \emph{Bernoulli}, pages 789--824, 2017.

\bibitem[Rockafellar(2015)]{rockafellar2015convex}
Ralph~Tyrell Rockafellar.
\newblock \emph{Convex Analysis:(PMS-28)}.
\newblock Princeton university press, 2015.

\bibitem[Rudelson and Vershynin(2006)]{rudelson2006combinatorics}
Mark Rudelson and Roman Vershynin.
\newblock Combinatorics of random processes and sections of convex bodies.
\newblock \emph{Annals of Mathematics}, pages 603--648, 2006.

\bibitem[Spielman and Teng(2004)]{spielman2004smoothed}
Daniel~A Spielman and Shang-Hua Teng.
\newblock Smoothed analysis of algorithms: Why the simplex algorithm usually takes polynomial time.
\newblock \emph{Journal of the ACM (JACM)}, 51\penalty0 (3):\penalty0 385--463, 2004.

\bibitem[Srebro et~al.(2010)Srebro, Sridharan, and Tewari]{srebro2010smoothness}
Nathan Srebro, Karthik Sridharan, and Ambuj Tewari.
\newblock Smoothness, low noise and fast rates.
\newblock \emph{Advances in neural information processing systems}, 23, 2010.

\bibitem[Syrgkanis et~al.(2016)Syrgkanis, Krishnamurthy, and Schapire]{syrgkanis2016efficient}
Vasilis Syrgkanis, Akshay Krishnamurthy, and Robert Schapire.
\newblock Efficient algorithms for adversarial contextual learning.
\newblock In \emph{International Conference on Machine Learning}, pages 2159--2168. PMLR, 2016.

\bibitem[Valiant(1984)]{valiant1984theory}
Leslie~G Valiant.
\newblock A theory of the learnable.
\newblock \emph{Communications of the ACM}, 27\penalty0 (11):\penalty0 1134--1142, 1984.

\bibitem[Van~Handel(2014)]{van2014probability}
Ramon Van~Handel.
\newblock Probability in high dimension.
\newblock \emph{Lecture Notes (Princeton University)}, 2014.

\bibitem[Vietri et~al.(2020)Vietri, Tian, Bun, Steinke, and Wu]{vietri2020new}
Giuseppe Vietri, Grace Tian, Mark Bun, Thomas Steinke, and Steven Wu.
\newblock New oracle-efficient algorithms for private synthetic data release.
\newblock In \emph{International Conference on Machine Learning}, pages 9765--9774. PMLR, 2020.

\bibitem[Wainwright(2019)]{wainwright2019high}
Martin~J Wainwright.
\newblock \emph{High-dimensional statistics: A non-asymptotic viewpoint}, volume~48.
\newblock Cambridge university press, 2019.

\bibitem[Yu et~al.(2022)Yu, Naik, Backurs, Gopi, Inan, Kamath, Kulkarni, Lee, Manoel, Wutschitz, et~al.]{yu2022differentially}
Da~Yu, Saurabh Naik, Arturs Backurs, Sivakanth Gopi, Huseyin~A Inan, Gautam Kamath, Janardhan Kulkarni, Yin~Tat Lee, Andre Manoel, Lukas Wutschitz, et~al.
\newblock Differentially private fine-tuning of language models.
\newblock In \emph{International Conference on Learning Representations (ICLR)}, 2022.

\bibitem[Zhou et~al.(2021)Zhou, Wu, and Banerjee]{ZhouW021}
Yingxue Zhou, Steven Wu, and Arindam Banerjee.
\newblock Bypassing the ambient dimension: Private {SGD} with gradient subspace identification.
\newblock In \emph{9th International Conference on Learning Representations, {ICLR} 2021, Virtual Event, Austria, May 3-7, 2021}. OpenReview.net, 2021.
\newblock URL \url{https://openreview.net/forum?id=7dpmlkBuJFC}.

\end{thebibliography}
